%% file: main.tex
\documentclass[11pt]{article}
\usepackage{microtype}  
\usepackage{fullpage}
\usepackage{url}    
\usepackage{booktabs}       
\usepackage{nicefrac}       
\usepackage{xcolor,colortbl} 
\usepackage{multirow}
\usepackage{tablefootnote}
\usepackage{subcaption}
\usepackage{makecell}
\usepackage{silence}
\usepackage{preamble}
\usepackage{mathnotations}

\usepackage{widebar}

\usepackage{arxiv}

\begin{document}
\title{\huge Towards Differentially Private Reinforcement Learning with General Function Approximation}

\date{}
\author{}

\maketitle

\vspace{-4em}

\begin{center}
    \large
    \setlength{\tabcolsep}{10pt}
    \begin{tabular}{ccc}
        Yi He$^\dagger$ &
        Xingyu Zhou$^\dagger$
    \end{tabular}
\end{center}

\vspace{1em}

\begingroup
\renewcommand\thefootnote{\fnsymbol{footnote}}
\footnotetext[2]{Wayne State University, Detroit, USA. Email: \texttt{\{yihe, xingyu.zhou\}@wayne.edu}}
\endgroup

\begin{abstract}
We present the first theoretical guarantees for differentially private online reinforcement learning (RL) with general function approximation, extending beyond prior work restricted to tabular and linear settings. Our approach combines a batched policy update scheme with the exponential mechanism, together with a novel regret analysis. We show that, even under general function approximation, the regret in the model-free setting under differential privacy matches the state of the art for the linear case, scaling as $\widetilde{O}(K^{3/5})$, where $K$ denotes the number of episodes. As an important by-product, we also establish the first regret bound for online RL with batch update that depends on the standard complexity measure of \emph{coverability}, complementing existing results based on a newly introduced Eluder-Condition class. In addition, we uncover fundamental gaps in recent results for private RL with linear function approximation, thereby clarifying its landscape.
\end{abstract}

\input{sections/sec-intro}

\input{sections/sec-pre}
\input{sections/sec-alg}
\input{sections/sec-conclusion}

\printbibliography
\newpage

\appendix
\renewcommand{\contentsname}{Contents of Appendix}
\addtocontents{toc}{\protect\setcounter{tocdepth}{2}}

{
  \tableofcontents
}

\input{sections/app-related}

\input{sections/app-infinite}
\input{sections/app-gaps}
\input{sections/app-experiments}

\input{sections/app-proof}

\end{document}

%% file: sections/sec-intro.tex
\section{Introduction}

Reinforcement learning (RL) provides a principled framework for learning sequential decision-making strategies in unknown environments through repeated interaction. By observing the consequences of its actions and adjusting its behavior accordingly, an RL agent can progressively improve its performance over time. This paradigm has become a key enabler of modern personalized systems, with successful applications ranging from online recommendation~\cite{li2010contextual} to individualized healthcare interventions~\cite{yu2021reinforcement}, social robotics~\cite{gordon2016affective}, and, more recently, large language models (LLMs) optimized via reinforcement learning from human feedback or verified reward~\cite{ouyang2022training,jaech2024openai}. In these applications, the agent engages with users over multiple rounds, choosing actions based on the current interaction context and receiving feedback---either as per-step process reward or per-trajectory outcome reward---which is then leveraged to refine the agent’s decision-making policy.

While highly effective, RL-based personalization inevitably raises privacy concerns, as the feedback collected during interaction often reflects sensitive user information. For example, in personalized medical decision-making, the state may encode private attributes such as a patient’s medical history or physiological measurements, while rewards reflect treatment effectiveness or adverse reactions. Similarly, in LLMs, user prompts, preference feedback, and interaction outcomes may reveal personal intentions, beliefs, or proprietary content. Although users may be willing to share such information to obtain better personalized services, they typically do not consent to having their private data inferred by third parties through the outputs or internal updates of the learning algorithm. Indeed, the existing non-private RL approaches have been shown to suffer from privacy leakages~\cite{pan2019you,feng2024exposing,liu2025grpo}. Thus, it is essential to develop RL methods that can learn from user interactions while providing formal guarantees on individual privacy.

Motivated by these concerns, a recent line of theoretical work has investigated how to incorporate differential privacy (DP)~\cite{dwork2006calibrating} into RL, with the goal of understanding the fundamental tradeoffs between privacy protection and sample efficiency (e.g.,~\cite{vietri2020private,chowdhury2022differentially,qiao2023near,zhou2022differentially,luyo2021differentially}). Since standard DP is incompatible with sublinear regret, this literature has largely focused on a relaxed notion named \emph{joint differential privacy} (JDP)~\cite{kearns2014mechanism}, under which meaningful guarantees become attainable. They have established sublinear regret bounds for various private RL algorithms, providing the first principled understanding of privacy–utility tradeoffs in RL. While pioneering, existing results are almost exclusively limited to tabular environments or linear function approximation. These structural assumptions are crucial in prior analyses, as they enable private learning through low-dimensional sufficient statistics, simple confidence sets, or carefully controlled policy switching. However, such assumptions substantially limit applicability in modern RL systems, where value functions or policies are typically represented using rich, non-linear function classes. Extending theoretical guarantees beyond the linear regime introduces significant new challenges, as many standard exploration mechanisms, such as confidence-set-based optimism, do not readily admit private implementations in the general function approximation case.

This gap naturally motivates our study of sample-efficient private RL with \emph{general function approximation}. In particular, we ask the following fundamental question:
\begin{center}
    \emph{Can we still enjoy sublinear regret for private RL with general function approximation?}
\end{center}
We answer this question in the affirmative by establishing the first sublinear regret bound for RL with general function approximation under JDP privacy constraint. Specifically, our main contributions can be summarized as follows.

\begin{enumerate}[leftmargin=*]
    \item We study online RL in episodic Markov decision processes (MDPs) with \emph{general function approximation} under formal privacy guarantees via joint differential privacy (JDP). We establish a regret bound of order $K^{3/5}$, where $K$ is the number of episodes, with additional dependence on the \emph{coverability} parameter that characterizes the exploration difficulty. This $K^{3/5}$-regret bound matches the state-of-the-art bound for the linear case in the model-free setting. Our results apply to both general stochastic MDPs (Theorem~\ref{thm:general}) and deterministic MDPs with outcome-based rewards (Theorem~\ref{thm:deter}), a modeling framework motivated by LLM applications, where a much weaker representation condition (i.e., realizability) is required.
    
    \item Our results are achieved via a carefully designed algorithmic framework built on several key ingredients (Algorithms~\ref{alg:general} and~\ref{alg:deterministic}). First, instead of the confidence-set-based constrained optimization commonly used for exploration in non-private RL, we adopt an unconstrained optimization objective that explicitly balances exploration and exploitation. This formulation enables the direct application of standard privacy mechanisms, such as the exponential mechanism, to privatize each policy update.  Second, to control the cumulative privacy loss incurred by repeated updates, we employ a fixed-size batching strategy with a balanced choice of batch size, updating the policy only at the beginning of each batch.
    
    \item In addition, our analysis yields several by-products of independent interest. First, even without privacy considerations, our regret analysis under fixed-size batching provides new guarantees for RL with adaptivity constraints, where the regret depends directly on \emph{coverability} (cf.~\eqref{eq:reg_decomp}), matching the standard exploration complexity measure and avoiding the need for newly introduced conditions in prior work. Second, our generic and modular regret analysis offers a simple recipe for designing private RL algorithms beyond the specific privacy mechanisms used in this paper. Finally, we clarify the state of the art for private RL with linear MDPs by identifying fundamental gaps in recent claims of improved regret bounds.
\end{enumerate}

\subsection{Related Work}
We mainly discuss the most relevant work here and relegate a detailed discussion to Appendix~\ref{app:related}. For private RL,~\citet{vietri2020private}, \citet{chowdhury2022differentially}, and \citet{qiao2023near} established increasingly sharp guarantees in tabular MDPs, while~\citet{zhou2022differentially} and \citet{luyo2021differentially} moved to linear function approximation in the model-based and model-free settings, respectively. In the non-private regime, RL with general function approximation has been studied through complexity measures such as coverability~\cite{xie2022rolecoverageonlinereinforcement}, as well as through unconstrained optimization frameworks~\cite{liu2023maximize,chen2025outcomebasedonlinereinforcementlearning} that are particularly relevant to our algorithmic design. A third relevant line studies adaptivity constraints for policy update~\cite{abbasi2011improved,xiong2023general}, including both data-independent batching and data-dependent rare policy switching, which will be useful for controlling the privacy composition in online RL.

%% file: sections/sec-pre.tex
\section{Preliminaries}

\subsection{RL with General Function Approximation}
\textbf{Markov decision processes (MDPs).} We study RL under a finite-horizon episodic Markov decision process specified by a tuple $ M=(\cS, \cA, H, P, R,
\rho)$, with state space $\cS$, action space $\cA$, horizon length $H$, transition kernel $P = \{P_h\}_{h=1}^H$, ${P}_h: \cS \times \cA \to \Delta(\cS)$, reward function $R = \{R_h\}_{h=1}^H$, $R_h:\cS \times \cA \to [0,1]$, and initial state distribution $\rho \in \Delta(\cS)$.  A (randomized) policy $\pi$ consists of $H$-step functions: $\pi = \{\pi_h: \cS \to \Delta(\cA)\}_{h=1}^H$. The policy induces a distribution over trajectories $(s_1, a_1, r_1),\ldots, (s_H, a_H, r_H)$ via the following process. The initial state $s_1$ is first drawn via $s_1 \sim \rho$, and then for each $h \in [H]$, $a_h \sim \pi(\cdot|s_h)$, $r_h = R_h(s_h, a_h)\footnote{For simplicity, we consider deterministic reward and our result can be easily extended to the random case.}$, $s_{h+1} \sim P_h(\cdot|s_h, a_h)$ and $s_{H+1}$ is the terminal state with zero reward. Accordingly, we let $\mathbb{E}_{\pi}[\cdot]$ and $\mathbb{P}_{\pi}[\cdot]$ denote expectation and probability under this process, respectively. 

The value function and $Q$-function for a policy $\pi$ at each step $h \in [H]$ are given by 
\begin{align*}
    V_{h}^{\pi}(s) &:= \mathbb{E}_\pi  \left[ \sum_{h'=h}^H r_{h'} \middle | s_h = s\right], \quad Q_{h}^{\pi}(s,a) := \mathbb{E}_\pi  \left[ \sum_{h'=h}^H r_{h'} \middle| s_h = s, a_h =a \right].
\end{align*}
The total expected reward under a policy $\pi$ is thus given by 
\begin{align*}
    J(\pi):= \mathbb{E}_{\pi}\left[\sum_{h=1}^H r_h\right] = \mathbb{E}_{s_1 \sim \rho} V_1^{\pi}(s_1).
\end{align*}
Let $\pi^{\star}$ denote an optimal policy (i.e., $\pi^{\star} \in \argmax_{\pi} J(\pi)$) and $V_h^{\star} := V_h^{\pi^{\star}}$, $Q_h^{\star}=Q_h^{\pi^{\star}}$ be the corresponding value function and $Q$-function, which satisfies the well-known Bellman equation for each $s \in \cS$, $a \in \cA$  and $h \in [H]$:
\begin{align*}
    V_h^{\star}(s) = \max_{a \in \cA} Q_h^{\star}(s,a), \quad Q_h^{\star}(s,a) = R_h(s,a) + \mathbb{E}_{s' \sim \mathbb{P}_h(\cdot|s,a)}V_{h+1}^{\star}(s').
\end{align*}
This motivates us to define the so-called Bellman operator $\cT_h$ as follows: for any $f: \cS \times \cA \to \Real$, $\cT_h f$ is defined as  
\begin{align*}
    [\cT_h f] (s,a) := R_h(s,a) +  \mathbb{E}_{s' \sim \mathbb{P}_h(\cdot|s,a)}\max_{a'\in \cA} f(s',a').
\end{align*}
Throughout the paper, as in previous papers (e.g.,~\cite{xie2022rolecoverageonlinereinforcement,chen2025outcomebasedonlinereinforcementlearning}), we assume that the rewards are normalized such that $\sum_{h=1}^H r_h \in [0,1]$.

\textbf{Online RL.} We study RL in the online setting where the learner repeatedly interacts with an unknown MDP $M^{\star}$ via taking a policy for each episode and observing the trajectory, to maximize the total reward. Formally, the learning protocol runs in $K$ episodes, where at each episode $k \in [K]$, the learner (i) selects a policy $\pi_k = \{\pi_{k,h}\}_{h=1}^H$ to run in unknown MDP $M^{\star}$ and (ii) observes the resultant trajectory $\{(s_h^{(k)}, a_{h}^{(k)}, r_{h}^{(k)})\}_{h=1}^H$. The performance of the learner is measured via the regret
\begin{align*}
    \mathsf{Reg}:= \sum_{k=1}^K J(\pi^{\star}) - J(\pi_k).
\end{align*}
\textbf{General function approximation.} To achieve sublinear regret guarantees without dependence on the size of $\cS$, we work with model-free value function approximation, where the learner has access to a  \emph{general value function class} $\cF = \cF_1 \times \cdots \times \cF_H$ that attempts to model the value functions under $M^{\star}$, with each $\cF_h \subseteq (\cS \times \cA \to [0,1])$. To derive non-trivial regret bounds for online RL, some structural conditions on the function class $\cF$ and its interplay with $M^{\star}$ are often required. The first set is the so-called \emph{representation conditions}, which require that $\cF$ can model some necessary value functions. Two common conditions in this set are \emph{realizability} and \emph{Bellman completeness}. 
\begin{assumption}[Bellman completeness and realizability] \label{ass:completeness}
    For all $h \in [H]$, we have $\cT_h f_{h+1} \in \cF_h$ for all $f_{h+1} \in \cF_{h+1}$. This directly implies realizability that $Q^{\star} \in \cF$.
\end{assumption}

In contrast to statistical learning, online RL requires another set of conditions, called \emph{exploration conditions}, which roughly limit the amount of exploration to learn the optimal policy. To this end, we adopt the notion of \emph{coverability}~\cite{xie2022rolecoverageonlinereinforcement}.
\begin{definition}[Coverability]
    The coverability constant $C_{\mathsf{cov}} >0$ for a given MDP $M$ and policy class $\Pi$ is 
\begin{align*}
    C_{\mathsf{cov}}(\Pi; M):= \min_{\mu_1, \cdots, \mu_H \in \Delta(\cS \times \cA)}
        \max_{h \in [H], \pi \in \Pi} \left\| \frac{d_h^\pi}{\mu_h} \right\|_\infty,
\end{align*}
 where $\left\| \frac{d_h^\pi}{\mu_h} \right\|_\infty := \max_{s \in \cS, a \in \cA} \frac{d_h^\pi(s,a)}{\mu_h(s,a)}$ and $d^{\pi}_h(s,a) := \mathbb{P}_{\pi}[s_h=s, a_h=a]$ is the occupancy measure of $\pi$ in $M$. For notation simplicity, we let $C_{\mathsf{cov}} := C_{\mathsf{cov}}(\Pi_{\cF}; M^{\star})$,
 where $\Pi_{\cF} = \{\pi_f \mid f \in \cF\}$ with $\pi_f$ being the greedy policy with respect to $f$ via $\pi_{f,h}(s) \in \argmax_{a\in \cA} f_h(s,a)$.
\end{definition}

With Bellman completeness,~\citet{xie2022rolecoverageonlinereinforcement} shows that low coverability is sufficient for sample-efficient online RL in the sense that sublinear regret, i.e., $\widetilde{O}(\sqrt{KC_{\mathsf{cov}}})$ is achieved. Our main goal is to establish the corresponding regret bound under formal privacy guarantees in RL, introduced below.

\subsection{Differential Privacy in RL}
To adapt standard DP~\cite{dwork2006calibrating} for RL, following prior work~\cite{vietri2020private,zhou2022differentiallyprivatereinforcementlearning}, the goal here is to roughly ensure that changing one trajectory in the learning process will not impact the outputs (e.g., actions and final policy) too much. More formally, we assume there is a user $u_k$ associated with each episode trajectory (e.g., a multi-step interaction between a user and LLMs) and write $U_K = (u_1, \ldots,u_K) \in \cU^K$ to denote a sequence of $K$ unique users in the process, where $\cU$ is the set of all users. A $k$-th neighboring sequence $U'_K$ differs from $ U_K$ in only one user at $k \in [K]$, i.e., $u_k \neq u'_k$ and $u_j = u'_j$ for $j \neq k$. The output of a learning mechanism $\cM$ on a user sequence $U_K$ is denoted as $\cM(U_K)$, which contains all the actions in the $K$ episodes as well as the final policy. With these, a natural adaptation of standard DP is to require that $ \cM(U_K) \approx_{\varepsilon,\delta} \cM(U'_K)$ for any neighboring sequences. However, this is problematic as changing a user $u_k$ will naturally change the corresponding actions recommended to her during episode $k$ (e.g., a different user with a different prompt for LLMs leads to different responses), thus failing to achieve meaningful DP guarantees. In fact, as shown in~\citet{shariff2018differentially}, even for contextual bandits (i.e., RL with $H=1$), a linear regret has to be sacrificed for meaningful DP guarantees. 

To resolve this conflict, as in prior work on private RL, we adopt the notion of \emph{joint differential privacy} (JDP)~\cite{kearns2014mechanism} from the privacy literature, which slightly relaxes the standard DP by excluding the actions corresponding to the differing user $u_k$, hence accommodating the intrinsic personalization in RL. Formally, we let $\cM_{-k}(U_K) := \cM(U_K) \setminus \{a_{h}^k\}_{h=1}^{H}$ denote the output obtained by excluding the actions during episode $k$ and define JDP in RL as follows as in~\citet{vietri2020private,zhou2022differentiallyprivatereinforcementlearning}.
\begin{definition}[JDP in RL]
    An RL learning algorithm $\cM$ is $(\varepsilon,\delta)$-jointly differentially private if for all $k \in [K]$, all $k$-th neighboring user sequences $U_K$, $U'_K$, and all events $E \subseteq \cA^{H \times [K-1]} \times \Pi$, we have
    \begin{align*}
    \mathbb{P}\left( \cM_{-k}(U_K)\in E \right)
    \le
    e^\varepsilon \cdot \mathbb{P} \left(\cM_{-k}(U'_K)\in E \right) + \delta,
\end{align*}
where $\Pi$ is the class for the final output policy.
\end{definition}
\begin{remark}
We remark that JDP remains a strong privacy guarantee: it protects the information associated with any individual user $u_k$ against arbitrary collusions of other users, provided that $u_k$ does not disclose the actions prescribed to her.
\end{remark}

Therefore, our ultimate goal is to design a private online RL algorithm that satisfies JDP while achieving a sublinear regret, which will be the main focus of the next section. 

%% file: sections/sec-alg.tex
\section{Sample Efficient Private Online RL}
\label{sec:alg}
In this section, we develop model-free private RL algorithms that achieve sublinear regret under the coverability condition. 

\subsection{General MDPs}

We present Algorithm~\ref{alg:general} for general MDPs with stochastic transition probability, which is the most standard setup in the non-private case\footnote{As in previous work, we assume that the initial state is fixed.}. As mentioned before, we take the unconstrained optimization rather than the confidence-based constrained optimization. 
Beyond this, two additional key algorithmic changes due to the privacy constraint are (i) batching policy update and (ii) exponential mechanism, which essentially replaces max with softmax. More specifically, at the beginning of each batch of size $B$, i.e., $k = i\cdot B + 1$ for some $i \ge 0$, it updates the policy via the following two steps. First, it computes the score function for each $f \in \cF$ by combining the Bellman loss 
$L_{\cD}^{\mathsf{BE}}(f)$ with an optimistic term $f_1(s_1):= \max_{a \in \cA} f_1(s_1, a)$ (the value of $f$ at the initial state), through a hyper-parameter $\eta >0$. In particular, the Bellman loss $L_{\cD}^{\mathsf{BE}}(f)$ is given by
\begin{align*}
    L_{\cD}^{\mathsf{BE}}(f):= \sum_{h =1}^H \cE_{\cD,h}(f_h, f_{h+1}) - \inf_{f' \in \cF} \sum_{h=1}^H \cE_{\cD,h}(f'_h, f_{h+1}),
\end{align*}
where Bellman error $\cE_{\cD,h}(f_h, f_{h+1})$ at each $h$ is given by 
\begin{align*}
  \sum_{\tau \in \cD} \left(f_h(s_h, a_h) - r_h - \max_{a'} f_{h+1}(s_{h+1},a')\right)^2.
\end{align*}

Note that the subtraction of the inf term in $L_{\cD}^{\mathsf{BE}}(f)$ is the standard approach to overcoming the double-sampling issue~\cite{antos2008learning,zanette2020learning,jin2021bellman}.
Leveraging the score function $S(f)$, the second step in the update is to sample $f^{(k)}$ from $\cF$ using the exponential mechanism~\cite{mcsherry2007mechanism}, i.e., each $f$ is sampled with probability proportional to $S(f)$ with temperature parameter $\beta_{\varepsilon,\delta}$, which is essentially a softmax rather than the standard max in the non-private case. Then, the updated policy is set to be the greedy policy with respect to $f^{(k)}$. For all other non-update episodes, the policy keeps the same as the previous episode. Finally, it runs the policy $\pi_k$ in the MDP, receives the new trajectory, consisting of $H$-step transition information, and updates dataset so far. 

\begin{algorithm}[!t]
  \caption{Private RL under General MDPs}
  \label{alg:general}
\begin{algorithmic}[1]
    \STATE {\bfseries Input:} $Q$-function class $\cF$, parameter $\eta > 0$,  batch size $B \in \mathbb{N}$, privacy parameters $\varepsilon, \delta$
    \STATE Initialize $\cD \leftarrow \varnothing$.
    \FOR{$k = 1, \ldots, K$}
    
        \IF{$k = i \cdot B + 1$ for some $i \ge 0$} 
            \STATE Compute the score function for each $f$ using $\cD$ :
            \begin{align*}
                 S(f) := \eta f_1(s_1) - L_{\cD}^{\mathsf{BE}}(f).
            \end{align*}
            \STATE Sample $f^{(k)} \in \cF$ via the Exponential Mechanism:
            \begin{align*}
                \mathbb{P} (f) \propto \exp \left(\beta_{\varepsilon,\delta}\cdot S(f) \right).
            \end{align*}
            \STATE Update policy $\pi_k \leftarrow \pi_{{f}^{(k)}}$.
        \ELSE
             \STATE Reuse: $\pi_k \leftarrow \pi_{k-1}$ and $f^{(k)} \leftarrow f^{(k-1)}$
        \ENDIF
        
        \STATE Execute $\pi_k$ for one episode and collect trajectory: $\tau^{(k)} = \{(s_h^{(k)}, a_h^{(k)}, r_h^{(k)})\}_{h=1}^H$.
        \STATE Update dataset: $\cD \leftarrow \cD \cup \{\tau^{(k)} \}$
    \ENDFOR
    \STATE {\bf return} $\hat{\pi} = \mathsf{Unif}(\pi_{1:K})$.
\end{algorithmic}
\end{algorithm}

\textbf{Optimistic exploration.} One key challenge in online reinforcement learning is strategic exploration. In tabular and linear settings, exploration is typically achieved via upper confidence bound (UCB)–type bonuses (e.g.,~\cite{azar2017minimax,ayoub2020model,jin2020provably}), which are also relatively easy to privatize since the required confidence estimates can be computed from low-dimensional sufficient statistics (e.g.,~\cite{vietri2020private,zhou2022differentiallyprivatereinforcementlearning,luyo2021differentially}). Under general function approximation, however, a common exploration strategy is to maintain a data-dependent confidence set (or version space) over candidate functions $f \in \cF$, as in \texttt{GOLF}~\cite{jin2021bellman}. Such confidence-set–based approaches are difficult to privatize, as they require repeatedly constructing and updating data-dependent confidence regions. As mentioned before, an alternative is to directly optimize a single objective that balances estimated value and statistical uncertainty, as in \texttt{MEX}~\cite{liu2023maximize}, an idea that dates back to~\citet{kumar1982}. This formulation is much more amenable to differential privacy, as it avoids explicit confidence sets and enables the use of standard private selection mechanisms such as the exponential mechanism (at least for finite function classes), which is the approach we adopt.

\textbf{Benefits of batching.} 
The key motivation for batching updates in the private setting is to control the composition of privacy loss arising from repeated operations on the same data. Without batching, data collected in the first episode would participate in all $K$ private updates, whereas batching limits this participation to only $O(K/B)$ updates. With an appropriate choice of the batch size $B$, this yields a sublinear rather than linear composition of privacy loss. Thus, tuning $B$ plays a central role in our analysis, as it governs the tradeoff between information freshness and cumulative privacy cost.

Algorithm~\ref{alg:general} enjoys the following regret bound, with its proof given by Appendix~\ref{app:general}.

\begin{theorem}
\label{thm:general}
   Let Assumption~\ref{ass:completeness} hold. Fix privacy parameters $\varepsilon,\delta$, there exist proper choices of $\beta_{\varepsilon,\delta}$, $\eta$ and $B$, such that Algorithm~\ref{alg:general} is $(\varepsilon,\delta)$-JDP while achieving the following regret bound\footnote{We use $\lesssim_{\delta}$ to ignore constant and $\log(1/\delta)$ privacy term for simplicity. See appendix for the detailed bound.} with probability at least $1-\alpha$
   \begin{align*}
       \mathsf{Reg} &\lesssim_{\delta} H\sqrt{C_{\mathsf{cov}} K \log(|\cF| KH/\alpha) \log K} + H \left(C_{\mathsf{cov}} K \log K\right)^{3/5}\cdot\left(\log(|\cF|K/\alpha)/\varepsilon\right)^{2/5}.
   \end{align*}
\end{theorem}

To the best of our knowledge, this is the first regret bound for private online RL.
Further, our algorithm is a model-free approach. We also remark that even for a finite function class $\cF$, our algorithm and its guarantees are useful in that they allow us to tackle meaningful problems that prior work cannot tackle due to the lack of structure (e.g., our experiment in Appendix~\ref{app:experiment}). Moreover, our result can also be extended to the infinite function class, see Appendix~\ref{app:infinite}.

\textbf{Comparison with prior work.} The first term in the above bound matches the one in the non-private case (cf. Theorem 1 in~\citet{xie2022rolecoverageonlinereinforcement}) while the second term is additional cost due to privacy, which is on the order of $K^{3/5}$. This privacy cost matches the state-of-the-art regret for RL with linear function approximation in the model-free setting~\cite{luyo2021differentially}, where a linear MDP as in~\citet{jin2020provably} is considered. On the other hand, we note that for a model-based approach under the linear mixture MDPs~\cite{ayoub2020model} (a different type of linear function approximation), one can achieve $\sqrt{K}$-regret bound under JDP as in~\citet{zhou2022differentiallyprivatereinforcementlearning}. The key difference here is that, instead of a $\mathsf{Poly}(K)$ privacy composition in the model-free setting, one can enjoy a total number of only $\log K$ compositions in the model-based linear mixture MDPs, thanks to the applicability of the binary tree mechanism~\cite{chan2011private} for privately aggregating the sufficient statistics. 
\begin{remark}
   We remark that some follow-up works~\cite{ngo2022improved,sahu2025towards} claim to achieve an improved $\sqrt{K}$-bound for private RL  in the model-free setting under linear MDPs, by reducing the number of compositions to $\log K$ via the technique of rare policy switch~\cite{abbasi2011improved}. However, as mentioned before, their analysis has gaps, making their results ungrounded, see Section~\ref{sec:dis}. 
\end{remark} 

As in the non-private case, one weakness of Theorem~\ref{thm:general} is its requirement of Bellman completeness, which is a much stronger condition than realizability. Moreover, the corresponding algorithm (Algorithm~\ref{alg:general}) requires solving a max-min optimization problem due to the inf operator in the Bellman loss $L_{\cD}^{\mathsf{BE}}(f)$. Also, for some interesting applications like LLMs, a general MDP with stochastic transitions is an overkill, as the transition could be a simple deterministic behavior. Finally, for LLMs, the learner typically has access only to an outcome reward, i.e., a single reward over the entire trajectory.  All of these facts lead us to the topic of the next section.

\subsection{Deterministic MDPs with Outcome Rewards}
We show that for deterministic MDPs with outcome rewards (Definition~\ref{def:det}), one can eliminate both the Bellman completeness condition and the max-min optimization, while still achieving $K^{3/5}$-regret bound under JDP. This can be viewed as the private counterpart of~\citet{chen2025outcomebasedonlinereinforcementlearning}.

\begin{definition}[Deterministic MDPs with outcome rewards]
\label{def:det}
    An MDP is said to be deterministic if its transition kernel P is deterministic, i.e., for any $h \in [H]$ and $s_h \in \cS$, $a_h \in \cA$, there is a unique state $s_{h+1} \in \cS$ such that $P_h(s_{h+1} | s_h, a_h) = 1$. An outcome reward is a single reward signal $r = \sum_{h=1}^H R_h(s_h, a_h)$ for a trajectory.
\end{definition}

A key consequence of deterministic transitions is that the Bellman equation holds almost surely, rather than only in expectation. Specifically, for any $(s,a)$ in step $h$
\begin{align*}
    Q_h^{\star}(s,a) &= R_h(s,a) + V_{h+1}^{\star}(s_{h+1}).
\end{align*}

This stands in sharp contrast to the stochastic setting, where the Bellman equation involves an expectation over next states. Hence, for any trajectory $\tau$, it holds that
\begin{align*}
    r(\tau) = \sum_{h=1}^H R_h(s_h, a_h) = \sum_{h=1}^H \left[ Q_h^{\star}(s_h, a_h) - V_{h+1}^{\star}(s_{h+1}) \right].
\end{align*}

For any $f\in\cF$, let $f_h(s):=\max_{a\in\cA}f_h(s,a)$ and $f_{H+1}\equiv0$. We define the outcome reward model $R^{(f)}:(\cS\times\cA)^H\to\mathbb{R}$ by
\begin{align*}
    R^{(f)}(\tau) := \sum_{h=1}^H \left[f_h(s_h, a_h) - f_{h+1}(s_{h+1}) \right].
\end{align*}

Motivated by the above observation, we replace the Bellman error loss used in general MDPs with the Bellman residual loss. Given a dataset $\cD$ consisting of collected trajectories, the loss is defined as
\begin{align*} 
    L_{\cD}^{\mathsf{BR}}(f) := \sum_{(\tau, r) \in \cD} \left( \sum_{h=1}^H \left[f_h(s_h, a_h) - f_{h+1}(s_{h+1}) \right] - r \right)^2.
\end{align*}

Unlike the stochastic case, this loss does not require subtracting an infimum over the function class to address the double-sampling issue. At each batch update, we combine this loss with the distributional optimistic value
\begin{align*}
    \bar{f}_1 := \mathbb{E}_{s_1 \sim \rho} \max_{a\in\cA} f_1(s_1,a),
\end{align*}
where $\rho$ is known as part of the MDP specification. This choice aligns the update rule with the regret objective $J(\pi)=\mathbb{E}_{s_1\sim\rho}V^\pi(s_1)$ and avoids making the score depend on a private sampled initial state.

\begin{algorithm}[!h]
  \caption{Private RL under Deterministic MDPs}
  \label{alg:deterministic}
\begin{algorithmic}[1]
    \STATE {\bfseries Input:} $Q$-function class $\cF$, parameter $\eta > 0$, batch size $B \in \mathbb{N}$, privacy parameters $\varepsilon, \delta$
    \STATE Initialize $\cD \leftarrow \varnothing$.
    \FOR{$k = 1, \ldots, K$}
        \IF{$k = i \cdot B + 1$ for some $i \ge 0$ }
            \STATE Compute the score function for each $f$ using $\cD$
            \begin{align*}
                S(f) := \eta \bar{f}_1 - L_{\cD}^{\mathsf{BR}}(f), \text{ where } \bar{f}_1 := \mathbb{E}_{s_1\sim\rho}\max_{a\in\cA}f_1(s_1,a).
            \end{align*}
            \STATE Sample $f^{(k)} \in \mathcal{F}$ with the Exponential Mechanism:
            \begin{align*}
                \mathbb{P} (f) \propto \exp \left(\beta_{\varepsilon,\delta}\cdot S(f) \right).
            \end{align*}
            \STATE Update policy $\pi_k \leftarrow \pi_{f^{(k)}}$.
        \ELSE
            \STATE Reuse: $\pi_k \leftarrow \pi_{k-1}$ and $f^{(k)} \leftarrow f^{(k-1)}$
        \ENDIF
        
        \STATE Execute $\pi_k$ for one episode and collect trajectory: $\tau^{(k)} = \{(s_h^{(k)}, a_h^{(k)})\}_{h=1}^H$ and outcome reward $r^{(k)}$.
        \STATE Update dataset: $\cD \leftarrow \cD \cup \{(\tau^{(k)}, r^{(k)})\}$
    \ENDFOR
    \STATE \textbf{return:} $\hat{\pi} \leftarrow \mathsf{Unif} \left(\pi_{1:K}\right)$.
\end{algorithmic}
\end{algorithm}

\textbf{Algorithm.} Compared with Algorithm~\ref{alg:general}, Algorithm~\ref{alg:deterministic} has only minor modifications: it replaces the Bellman-error loss $L_{\cD}^{\mathsf{BE}}(f)$ with the outcome Bellman-residual loss (i.e., $L_{\cD}^{\mathsf{BR}}(f)$ above) computed from trajectory-level rewards, and uses the distributional optimistic term $\bar{f}_1$ in the score function. Consequently, it avoids the infimum correction in $L_{\cD}^{\mathsf{BE}}(f)$, and thus removes the max-min optimization while requiring only realizability rather than Bellman completeness.
Algorithm~\ref{alg:deterministic} enjoys the following regret bound, which scales with the following
coverability
\begin{align*}
    C'_\mathsf{cov} = \mathbb{E}_{s_1 \sim \rho} C_{\mathsf{cov}}(\Pi_{\cF};M_{s_1}^\star),
\end{align*}
where $M_{s_1}^{\star}$ is the MDP with deterministic initial state $s_1$ and the same transition
dynamics as $M^\star$. In general, $C'_\mathsf{cov}$ is always an upper bound on the coverability $C_\mathsf{cov}$. The detailed proof is given by Appendix~\ref{app:deter}.

\begin{theorem}
\label{thm:deter}
    Suppose that realizability holds (s.t. $Q^\star \in \cF$), and the initial-state distribution $\rho$ is known as part of the public MDP specification. Fix privacy parameters $\varepsilon,\delta$, there exist proper choices of $\beta_{\varepsilon,\delta}$, $\eta$ and $B$, such that Algorithm~\ref{alg:deterministic} is $(\varepsilon,\delta)$-JDP while achieving the following regret bound with probability at least $1-\alpha$
   \begin{align*}
       \mathsf{Reg} &\lesssim_{\delta}
        H^{2}\sqrt{C'_{\mathsf{cov}}K\log(|\cF|K/\alpha)\log K} + H^{7/5}\left(C'_{\mathsf{cov}}K\log K\right)^{3/5}
        \left( \log(|\cF|K/\alpha) / \varepsilon\right)^{2/5}
   \end{align*}
\end{theorem}
The first term matches the one in~\citet{chen2025outcomebasedonlinereinforcementlearning} while the second term is the privacy cost of order $K^{3/5}$.
\begin{remark}
    We remark that stochastic initial states are essential in the deterministic setting considered here. With a fixed initial state and deterministic transitions, the interaction can become degenerate; using $s_1\sim\rho$ leads to a meaningful regret objective and is consistent with the distributional optimistic term used by Algorithm~\ref{alg:deterministic}.
\end{remark}

\begin{remark}
    To illustrate our algorithms and show privacy-utility tradeoff, we include a small proof-of-concept simulation in Appendix~\ref{app:experiment}. The experiment is intentionally minimal and is not meant as a comprehensive empirical study; rather, it demonstrates in a controlled deterministic outcome-reward setting that batching and privacy noise introduce the expected slowdown, while larger privacy parameters lead to faster stabilization. We note that even for this simple MDP, prior work on tabular or linear private RL cannot be applied, demonstrating the usefulness of our algorithm.
\end{remark}

\subsection{Proof Sketch}
\label{sec:ps}
We present a proof sketch for Theorem~\ref{thm:general} while also mentioning the key difference in the proof for Theorem~\ref{thm:deter}.

\textbf{Privacy analysis.} By the \emph{Billboard Lemma}~\cite{hsu2016}, for $(\varepsilon,\delta)$-JDP guarantees, it suffices to show that the computation of the sequence of $f^{(k)}$ satisfies $(\varepsilon,\delta)$-DP in the standard sense. Since there is a total of $M= K/B$ updates, by the advanced composition theorem~\cite{dwork2010boosting}, for a final $(\varepsilon,\delta)$-DP, it suffices to ensure that each update is $(\varepsilon_0,0)$-DP with $\varepsilon_0 = \varepsilon/(2\sqrt{2M \log(1/\delta)})$. Thus, setting $\beta_{\varepsilon,\delta} = \frac{\varepsilon_0}{2\Delta_S}$ in the exponential mechanism gives the result, where $\Delta_S = O(H)$ is the sensitivity of our score function $S(f)$.

\textbf{Regret analysis.}
We start with the regret decomposition (for the fixed initial state assumed for Theorem~\ref{thm:general}). 
\vspace{-0.5em}
\begin{align*}
    \mathsf{Reg}=  \underbrace{\sum_{k=1}^K V_1^{\star}(s_1) - f_1^{(k)}(s_1)}_{\mathsf{T}_1} +  \underbrace{ \sum_{k=1}^K f_1^{(k)}(s_1) - V_1^{\pi_k}(s_1)}_{\mathsf{T}_2}.
\end{align*}
To bound $\mathsf{T}_1$, we first \emph{suppose that $f^{(k)}$ is obtained via max rather than softmax}. Then, we have 
\vspace{-0.5em}
\begin{align*}
    \eta f_1^{(k)}(s_1) - L_{\cD_{[k]}}^{\mathsf{BE}}(f^{(k)}) \ge \eta V_1^{\star}(s_1) - L_{\cD_{[k]}}^{\mathsf{BE}}(Q^{\star}),
\end{align*}
where we let $[k]$ denote the starting point of the batch that $k$ belongs to and let $\cD_{[k]}$ be the dataset at the start of episode $[k]$. That is, if $i \cdot B + 1 \le k < (i+1) \cdot B + 1$ for some $i \ge 0$, then $[k] = i \cdot B + 1$ and $\cD_{[k]}$ contains all episodes from $1$ to $i\cdot B$. The above inequality holds directly by the batching schedule and the updated rule, provided $f^{(k)}$ is the maximum of the score function. However, due to the softmax of the exponential mechanism, we only have approximate maximization, thus giving that w.p. $1-\alpha$,
\vspace{-0.5em}
\begin{align*}
    \eta f_1^{(k)}(s_1) - L_{\cD_{[k]}}^{\mathsf{BE}}(f^{(k)}) \ge \eta V_1^{\star}(s_1) - L_{\cD_{[k]}}^{\mathsf{BE}}(Q^{\star}) - \mathsf{Err}_{\mathsf{priv}},
\end{align*}
where $\mathsf{Err}_{\mathsf{priv}}$ is the privacy cost, and in particular for the exponential mechanism   $\mathsf{Err}_{\mathsf{priv}} = O(1/\beta_{\varepsilon,\delta}\log(|\cF|K/\alpha))$. Thus, re-arranging the above inequality, we bound $\mathsf{T}_1$ as 
\vspace{-0.5em}
\begin{align*}
    \mathsf{T}_1 \le \sum_{k=1}^K  \frac{1}{\eta} \cdot \left(L_{\cD_{[k]}}^{\mathsf{BE}}(Q^{\star}) - L_{\cD_{[k]}}^{\mathsf{BE}}(f^{(k)}) + \mathsf{Err}_{\mathsf{priv}}\right).
\end{align*}
\vspace{-0.1em}

Further, by the standard uniform convergence bound (cf.~\citet{jin2021bellman}), we can bound the empirical difference by the expectation of the Bellman error plus some statistical error, under the Bellman completeness assumption (Assumption~\ref{ass:completeness}). With this result, we have the final high probability bound for $\mathsf{T}_1$ as
\vspace{-0.5em}
\begin{align}
    \mathsf{T}_1 &\lesssim 
    \frac{1}{\eta} \left( K \cdot \mathsf{Err}_{\mathsf{stat}} + K \cdot \mathsf{Err}_{\mathsf{priv}}\right) - \frac{1}{\eta}{\left({\sum_{h=1}^H \sum_{k=1}^K \sum_{t < [k]}\mathbb{E}_{\pi_t}[e_h^{(k)}(s_h, a_h)^2]}\right)}\label{eq:empirical},
\end{align}

where $\mathsf{Err}_{\mathsf{stat}}$ is the statistical error, which can be bounded by $H \log(|\cF|KH/\alpha)$ for a finite class $\cF$ and $e_h^{(k)}(s, a):= f_h^{(k)}(s, a) - [\cT_h f_{h+1}^{(k)}(s,a)]$.

We now turn to $\mathsf{T}_2$, i.e., the on-policy error. First, by the standard performance difference lemma~\cite{kakade2002approximately}, we have that
\vspace{-0.5em}
\begin{align*}
    \mathsf{T}_2 = \sum_{h=1}^H \sum_{k=1}^K \mathbb{E}_{\pi_k}[e_h^{(k)}(s_h, a_h)].
\end{align*}
\vspace{-0.1em}
The ultimate goal is to relate such ``test'' error above to the ``training error'', i.e.,~\eqref{eq:empirical}, through the notion of coverability. This is the key challenge in our analysis due to batching, where the standard result (e.g.,~\cite{chen2025outcomebasedonlinereinforcementlearning}) cannot be applied. To this end, we derive a new inequality for such a key step in the batching setting. Specifically, we establish that for any $\lambda \ge 1$
\vspace{-0.5em}
\begin{align*}
    \sum_{k=1}^K \mathbb{E}_{\pi_k}\abs{e_h^{(k)}(s_h, a_h)} &\lesssim  C_{\mathsf{cov}} \cdot B \cdot \log K + \sqrt{C_{\mathsf{cov}}\log K \cdot \left(K \lambda  + \sum_{k=1}^K \sum_{t < [k]}\mathbb{E}_{\pi_t}[e_h^{(k)}(s_h, a_h)^2]\right) },
\end{align*}

where the first term is the cost due to batch update, and the second term is related to the negative term in $\mathsf{T}_1$ (cf.~\eqref{eq:empirical}), allowing for the cancellation effect.

Finally, by Cauchy-Schwarz inequality, plugging the above bound into $\mathsf{T}_2$, and merging with $\mathsf{T}_1$, we have
\vspace{-0.3em}
\begin{align} \label{eq:reg_decomp}
    \mathsf{Reg} \lesssim \underbrace{(1/\eta)\cdot  K \cdot \mathsf{Err}_{\mathsf{stat}} + \eta \cdot H \cdot C_{\mathsf{cov}}\log K}_{\text{(a)}}
     + \underbrace{H \cdot C_{\mathsf{cov}} \cdot B \cdot \log K }_{\text{(b)}} + \underbrace{(1/\eta)\cdot  K \cdot \mathsf{Err}_{\mathsf{priv}}}_{\text{(c)}}
\end{align}

where (a) is the same term as in the non-private case, (b) is the cost due to batching and (c) is the cost due to privacy, which depends on $B$. Thus, tuning $B$ and $\eta$ optimally gives us the final regret bound.

\begin{remark}
 With the above final generic regret bound, one can see that without privacy, by setting $B = O(\sqrt{K})$, the regret bound still maintains as $\sqrt{K}$ with dependence on the coverability. Hence, as a by-product, our result also gives a $\sqrt{K}$ regret with $\sqrt{K}$ update frequency under the coverability, in contrast to the new notion of Eluder-Condition class in~\citet{xiong2023general}.  
\end{remark}

\begin{remark}[Difference in the proof of Theorem~\ref{thm:deter}]
The proof essentially follows from the same flow as in the general MDPs for Theorem~\ref{thm:general}. The key technical difference is a new inequality to relate the on-policy ``test error'' to the ``training'' error, by fully utilizing the deterministic nature, which is our key technique contribution. As in the general MDPs, we believe that this inequality can find broad applications. Another difference is the treatment of the initial state: the algorithm uses the distributional optimistic term $\bar f_1=\mathbb{E}_{s_1\sim\rho}\max_a f_1(s_1,a)$ rather than the value at a sampled initial state, and the proof conditions on a fixed $s_1$ before averaging over $s_1\sim\rho$. Finally, the Bellman-error loss is replaced by the outcome Bellman-residual loss $L_{\cD}^{\mathsf{BR}}(f)$, so the proof requires only realizability rather than Bellman completeness, while the score sensitivity becomes $O(H^2)$.
\end{remark}

%% file: sections/sec-conclusion.tex
\section{Concluding Remarks}

\label{sec:dis}

\textbf{General MDPs with outcome rewards.} In this setting, we can easily extend our algorithm and analysis to derive a $K^{3/5}$-regret bound under JDP. In particular, following the non-private case~\cite{chen2025outcomebasedonlinereinforcementlearning}, we add another loss for reward model estimation to our Algorithm~\ref{alg:general} and modify the way of data collection with a helper reference policy. These changes will not impact the final regret order. In addition, if $\cF$ is parametric with parameter $\theta \in \Real^d$, one can potentially bound the term $\mathsf{Err}_{\mathsf{priv}}$ via DP-ERM for the empirical loss $F_{\cD}(\theta):= \eta f_{\theta,1}(s_1) - L_{\cD}^{\mathsf{BE}}(f_{\theta})$, thus bridging with the literature on private (minimax) optimization (e.g.,~\cite{chaudhuri2011differentially,bassily2014private,zhang2022bring}).

\textbf{Coverability vs Eluder condition.} Instead of using coverability, another related approach is based on Bellman-Eluder dimension, such as the Eluder-Condition class of~\citet{xiong2023general}, which provides a broad framework beyond several standard structural models. However, Bellman-Eluder-dimension-based analyses can be suboptimal in some cases, yielding a $T^{2/3}$ regret rate even when a $\sqrt{T}$ bound is achievable under coverability, see Proposition 11 in~\citet{xie2022rolecoverageonlinereinforcement}. Going beyond coverability may require modifying the $\mathsf{SEC}$ of~\citet{xie2022rolecoverageonlinereinforcement}, in a way similar to how~\citet{xiong2023general} modified the Bellman-Eluder dimension; we leave this as an interesting open problem.

\textbf{Infinite function class.} Our general regret bound in Section~\ref{sec:ps} provides us with some guidelines for dealing with an infinite function class $\cF$. Specifically, we only need to focus on the two error terms $\mathsf{Err}_{\mathsf{stat}}$ and $\mathsf{Err}_{\mathsf{priv}}$. For $\mathsf{Err}_{\mathsf{stat}}$, one can follow the standard approach in the non-private case to adopt a covering number argument, replacing the $|\cF|$ by the covering number (with some approximation error). For $\mathsf{Err}_{\mathsf{priv}}$, one can also potentially \emph{explicitly} construct a covering set in the algorithm and apply the exponential mechanism on the finite covering set. Due to the page limit, we give a formal guarantee in Appendix~\ref{app:infinite} for more details.

\textbf{Gaps in private RL with linear MDPs.} In both works~\cite{ngo2022improved,sahu2025towards}, the authors adopt the idea of the rare policy switch via the standard determinant trick argument over the design matrix. This is a \emph{data-dependent adaptive} update rather than the fixed-size batching in our paper, hence enabling only a total of $\log K$ updates, which further reduces the privacy composition. However, due to its data-dependent nature, one also has to privatize the design matrix. This leads to the gap in their analysis, as the standard condition for the determinant trick fails for a private design matrix. In particular, ~\citet{ngo2022improved} incorrectly handles monotonicity in both Lemma 31, concerning log-order switches, and Equation 20, when applying the standard determinant trick ~\cite{abbasi2011improved}. These errors subsequently propagate to the follow-up work by ~\citet{sahu2025towards}, see Appendix~\ref{app:gaps} for more detail.

\textbf{Conclusion.} In this paper, we established the first theoretical guarantees for private online RL with general function approximation. By combining a batched policy update scheme with an unconstrained exploration objective and the exponential mechanism, we showed that sublinear regret remains achievable under JDP, with regret scaling as $K^{3/5}$ and depending on the standard notion of coverability. Our analysis not only matches the best known model-free bounds for linear function approximation, but also yields new regret guarantees for RL with adaptivity constraints even in the non-private setting. We also clarified fundamental gaps in recent claims of improved regret bounds for private RL with linear MDPs. We believe our framework provides a principled foundation for future work on private RL with rich function classes, outcome rewards, and alternative privacy mechanisms.

%% file: sections/app-related.tex
\section{Detailed Related Work}
\label{app:related}

\textbf{Private RL.} 
Under the central trust model of joint differential privacy (JDP), the first theoretical guarantees for private reinforcement learning were established by~\citet{vietri2020private} in the tabular setting, via privatized \texttt{UCB-VI}–type algorithms (e.g.,~\cite{azar2017minimax}). Subsequently,~\citet{chowdhury2022differentially} improved the regret bounds and introduced the first policy-based (as opposed to value-based) private tabular RL algorithm. The current state of the art in the tabular regime is due to~\citet{qiao2023near}, which leverages Bernstein-type exploration bonuses.

Moving beyond tabular environments, existing theoretical results for private RL are largely restricted to linear function approximation. In particular,~\citet{zhou2022differentially} obtained $\sqrt{K}$ regret bounds for private RL with linear function approximation in the model-based setting via linear mixture MDPs~\cite{ayoub2020model}, while~\citet{luyo2021differentially} provided the first model-free result for linear MDPs~\cite{jin2020provably}, achieving a regret bound of order $K^{3/5}$. Motivated by this separation, several recent works~\cite{ngo2022improved,sahu2025towards} revisit the model-free setting and claim to close the gap by deriving $\sqrt{K}$ regret bounds. However, we identify a common and fundamental gap in their analyses (cf. Section~\ref{sec:dis}), rendering these improved bounds ungrounded. As a result, for private RL with linear function approximation, the best known regret guarantee in the model-free setting remains $K^{3/5}$. We also note in passing that a parallel line of work studies private RL under the stronger local differential privacy model—without a trusted central learner—which is beyond the scope of this paper~\cite{garcelon2021local,chowdhury2022differentially,liao2023locally}.

\textbf{RL with general function approximation.} 
In the non-private setting, substantial recent progress has been made in understanding exploration complexity beyond linear function approximation. Prominent complexity measures include the Eluder dimension~\cite{wang2020reinforcement,russo2013eluder,osband2014model}, the Bellman-Eluder dimension~\cite{jin2021bellman}, Bellman rank~\cite{jiang2017contextual}, witness rank~\cite{sun2019model}, bilinear classes~\cite{du2021bilinear}, and decision-estimation coefficients (DEC)~\cite{foster2021statistical}. Among these,~\citet{xie2022rolecoverageonlinereinforcement} introduced the notion of \emph{coverability} and showed that it suffices to guarantee sample-efficient online RL with general function approximation.

On the algorithmic side, while most prior works rely on constrained optimization over confidence sets,~\citet{liu2023maximize} proposed an unconstrained optimization framework and established sublinear regret guarantees in terms of the generalized Eluder coefficient (GEC)~\cite{zhong2022gec}. This complexity measure is closely related to the sequential extrapolation coefficient (SEC), introduced by~\citet{xie2022rolecoverageonlinereinforcement} as a generalization of coverability. While these works typically assume access to per-step process rewards,~\citet{chen2025outcomebasedonlinereinforcementlearning} recently took the first step toward studying RL with general function approximation under outcome rewards, relying on a similar unconstrained optimization approach in~\citet{liu2023maximize} with additional algorithmic changes for outcome reward.

\textbf{RL with adaptivity constraints.} 
This line of work studies RL under restrictions on how frequently the policy can be updated, motivated by practical concerns such as policy switching costs. A simple and data-independent approach is fixed-size batching, where the policy is updated only at the beginning of each batch. This setting has been extensively studied in bandits (e.g.,~\cite{gao2019batched,han2020sequential}) and later extended to RL with linear function approximation~\cite{wang2020reinforcement}. More recently,~\citet{xiong2023general} analyzed batching under general function approximation and established sublinear regret guarantees for a newly introduced class of MDPs termed the Eluder-Condition.
An alternative approach is data-dependent rare policy switching, dating back to~\citet{abbasi2011improved} in the bandit setting. Compared to fixed-size batching, rare policy switching typically enables $\sqrt{K}$ regret with only $O(\log K)$ policy updates, rather than $O(\sqrt{K})$. Such results have been obtained for tabular~\cite{baiyu2019lowswitchcosttabular}, linear~\cite{quanquangu2021linearlowcost}, and general function approximation settings~\cite{kong2021online,qiao2023logarithmic,xiong2023general}. 
Both batching and rare policy switching have also been employed in private bandits and private RL to control privacy loss composition. For example,~\citet{chowdhury2022shuffle} applies batching to study shuffle privacy in linear contextual bandits, and~\citet{luyo2021differentially} used it for private RL with linear function approximation. Meanwhile, rare policy switching with $O(\log K)$ updates is the primary technique used in~\cite{ngo2022improved,sahu2025towards} to claim $\sqrt{K}$ regret bounds for private RL; however, due to additional privacy noise, these analyses are flawed, as discussed in detail in Section~\ref{sec:dis}.

%% file: sections/app-infinite.tex
\section{Detailed Discussion for Infinite Function Class} \label{app:infinite}

To handle an infinite function class, we define the covering number as follows:
\begin{definition}
    Let $\cF \subseteq \{ f = (f_1, \ldots, f_H): f_h: \cS \times \cA \rightarrow \mathbb{R} \}$, and define
    \begin{align*}
        \| f-g \|_\infty := \max_{h \in [H]} \sup_{(s,a) \in \mathcal{S} \times \mathcal{A}} |f_h(s,a) - g_h(s,a)|.
    \end{align*}
    For $\beta \ge 0$, a set $\cC_\beta \subseteq \cF$ is a $\beta$-cover of $\mathcal{F}$, if for every $f \in \mathcal{F}$, there exists $\tilde{f} \in \mathcal{C}_\beta$ such that $\| f - \tilde{f} \|_\infty \le \beta$. The $\beta$-covering number is
    \begin{align*}
        \mathcal{N}(\mathcal{F}, \beta) := \min \{ |\mathcal{C}_\beta|: \mathcal{C}_\beta \text{ is a } \beta \text{-cover of } \mathcal{F} \}.
    \end{align*}
\end{definition}

\textbf{Algorithm.}  To handle an infinite function class, we only need to change the \textbf{input} of Algorithm~\ref{alg:general}, specifically, fix $\beta > 0$, explicitly construct a finite $\beta$-cover $\mathcal{C}_\beta \subseteq \mathcal{F}$ and replace $\mathcal{F}$ by $\mathcal{C}_\beta$ in Algorithm~\ref{alg:general}. Hence, in each update episode $k = i B + 1$, we compute scores on this finite $\mathcal{C}_\beta$ instead of infinite $\mathcal{F}$, which allows us to run the exponential mechanism.

Thus, we can derive a formal corollary:
\begin{corollary}[Infinite class reduction via explicit covering]
    Assume the same conditions as Theorem~\ref{thm:general}, except that $\mathcal{F}$ can be infinite. Fix $\beta > 0$ and run the above revised algorithm with an explicit $\beta$-cover $\mathcal{C}_\beta$, there exist proper choices of $\beta_{\varepsilon,\delta}$ , $\eta$ and $B$, such that the revised algorithm is $(\varepsilon,\delta)$-JDP while achieving the following regret bound with probability at least $1- \alpha$
    \begin{align*}
        \mathsf{Err}_{\mathsf{stat}} \lesssim \log \left( \frac{\mathcal{N}(\mathcal{F}, \beta)KH}{\alpha} \right) \quad \mathsf{Err}_{\mathsf{priv}} \lesssim \log \left( \frac{\mathcal{N}(\mathcal{F}, \beta)K}{\alpha} \right)
    \end{align*}
    Consequently,
    \begin{align*}
        \mathsf{Reg} &\lesssim_{\delta} H\sqrt{C_{\mathsf{cov}}K \log \left( \frac{\mathcal{N}(\mathcal{F}, \beta)KH}{\alpha} \right)\log K} \\ 
        &\quad + H(C_{\mathsf{cov}}K \log K)^{3/5}\cdot\left(\frac{\log(\mathcal{N}(\mathcal{F}, \beta)K/\alpha)}{\varepsilon}\right)^{2/5} + \beta K.
    \end{align*}
\end{corollary}

\begin{remark}
    The first two terms are exactly the finite-class bound with $|\mathcal{F}|$ replaced by $\mathcal{N}(\mathcal{F}, \beta)$, and $\beta K$ is the additional approximation error over $K$ episodes from restricting the Exponential Mechanism to sample exclusively from $\mathcal{C}_\beta$.
\end{remark}

%% file: sections/app-gaps.tex
\section{Gaps in Private RL with Linear MDPs} \label{app:gaps}

In this section, we further clarify the gaps in prior work.

As discussed in Section~\ref{sec:dis}, because rare policy switching is data-dependent, applying the standard determinant-trick argument requires privatizing the design matrix. This introduces Gaussian privacy noise, which is not a non-negative rank-one update. Consequently, the monotonicity property need not hold, making the follow-up arguments ungrounded.

Monotonicity of the design matrix is required for two claims: (a) \textit{log-order switches} (Lemma 31 in~\citet{ngo2022improved}, key for privacy accounting), and (b) \textit{norm factor bound} (Eq. 20 in~\citet{ngo2022improved}), namely that the norms induced by two design matrices are within a constant factor of each other.

For (a), we dive into Lemma 31 in~\citet{ngo2022improved}, where the determinant doubling trick is adopted. In particular, without monotonicity, the following step in their proof would fail:
\begin{align*}
    \frac{\prod_{h=1}^H \det(\tilde{\Lambda}_h^{k_i})}{\prod_{h=1}^H \det(\tilde{\Lambda}_h^{k_{i-1}})} = \frac{\det(\tilde{\Lambda}_{h^\star}^{k_i})}{\det(\tilde{\Lambda}_{h^\star}^{k_{i-1}})} \times \prod_{h \neq h^*} \frac{\det(\tilde{\Lambda}_h^{k_i})}{\det(\tilde{\Lambda}_h^{k_{i-1}})} > 2
\end{align*}

To see this, note that the first term is greater than 2 by definition of the switching design (triggered by a particular $h^{\star} \in [H]$, see line 6 of Algorithm 1 in~\citet{ngo2022improved}), but the second term may be less than 1 if the design matrix is not monotone. This indeed can happen because their Algorithm 1 will add fresh privacy noise (via Binary tree mechanism) to each design matrix at each $h \in [H]$ (see line 20 of Algorithm 1 in~\citet{ngo2022improved}), which can make the private design matrix non-monotone, even though it is still PSD.

For (b), the norm factor bound (Eq. 20 in~\citet{ngo2022improved}) fails because the standard determinant trick (Lemma 12 in~\citet{abbasi2011improved}) strictly requires monotonicity. Their analysis relies on the following inequality holding for the current private matrix $\tilde{\mathbf{\Lambda}}_h^k$ and the last-updated matrix $\tilde{\mathbf{\Lambda}}_h^{\tilde{k}}$:
\begin{align*}
    (\phi_h^k)^\top \left( \tilde{\mathbf{\Lambda}}_h^{\tilde{k}} \right)^{-1} \phi_h^k \leq 2(\phi_h^k)^\top \left( \tilde{\mathbf{\Lambda}}_h^k \right)^{-1} \phi_h^k
\end{align*}

However, as noted in (a), Algorithm 1 (line 20) in~\citet{ngo2022improved} injects fluctuating noise via the Binary tree mechanism at every episode. This inherently breaks monotonicity ($\tilde{\mathbf{\Lambda}}_h^k \nsucceq \tilde{\mathbf{\Lambda}}_h^{\tilde{k}}$), meaning $\tilde{\mathbf{\Lambda}}_h^k$ can shrink along certain feature directions. Without monotonicity, the inverse matrix relationship breaks down. Instead of being bounded by a constant factor of $2$, the ratio can blow up exponentially in dimension $d$, fundamentally invalidating the claimed $\tilde{O}(\sqrt{K})$ regret bound. 

%% file: sections/app-experiments.tex
\section{Experiments}
\label{app:experiment}

While our main focus is to provide theoretical understandings of DP-RL, rather than extensive empirical studies, we conducted a small \textbf{proof-of-concept} simulation in a deterministic outcome-reward setting designed to mirror the structure of Algorithm~\ref{alg:deterministic} while remaining simple enough to implement exactly. All experiments were run on a standard CPU machine and took less than a few minutes; no GPU was used.

\textbf{Environment construction.} We consider a finite-horizon deterministic episodic MDP with horizon $H=4$, setting a binary action space $\mathcal{A}=\{0,1\}$. At the beginning of each episode, a binary context
\begin{align*}
    x=(x_1,\dots,x_6)\in\{0,1\}^6
\end{align*}
is sampled uniformly at random. Thus, the state at stage $h$ is $s_h=(x,a_1,\dots, a_{h-1},h)$, so the transition is deterministic, since the next state is obtained by appending the chosen action.

There is no intermediate reward; the environment provides only a terminal outcome reward $r\in\{0,1\}$ at the end of the episode, which is determined by a hidden target hypothesis. 

An episode receives a reward of $1$ if the full action sequence matches the target sequence prescribed by the hidden hypothesis for the realized context, and $0$ otherwise. Thus, this is a deterministic outcome-reward MDP.

\textbf{Finite nonlinear hypothesis class.} We use a finite nonlinear hypothesis class with $|\cF|=243$. Each hypothesis consists of a binary gate $g(x)\in\{0,1\}$, chosen from
\begin{align*}
    g_0(x)&=1,\\
    g_1(x)&=x_5\oplus x_6,\\
    g_2(x)&=x_1\oplus x_3\oplus x_5,
\end{align*}
and one \textbf{stagewise action rule} for each stage $h=1,2,3,4$, chosen from
\begin{align*}
    u^{(0)}(x,\text{prefix})&=x_1\oplus x_2,\\
    u^{(1)}(x,\text{prefix})&=x_3\oplus a_{h-1}\quad (a_0:=0),\\
    u^{(2)}(x,\text{prefix})&=\mathrm{parity}(a_1,\dots,a_{h-1})\oplus x_4.
\end{align*}
Hence, each hypothesis defines a nonlinear, context-dependent, and history-dependent target action sequence.

\textbf{Easy and hard instances.} We report results on two instances. In the \textbf{easy environment}, the hidden target gate is $g_0(x)=1$, and the stage rules are $(u^{(0)},u^{(1)},u^{(0)},u^{(1)})$. In this case, every context is rewardable, so positive outcomes are relatively dense.

In the \textbf{hard environment}, the hidden target gate is $g_2(x)=x_1\oplus x_3\oplus x_5$, and the stage rules are $(u^{(2)},u^{(1)},u^{(2)},u^{(1)})$. In this case, only a subset of contexts are rewardable, and the correct sequence depends more strongly on action-history parity, making learning more difficult.

\textbf{Methods compared.} We compare four methods: 
\begin{itemize}
    \item \texttt{non\_private\_no\_batch}, a non-private method updated after every episode; 
    \item \texttt{non\_private\_batched}, a non-private method restricted to batched updates; 
    \item \texttt{private\_eps\_8}, a private batched selection method with privacy/noise parameter $\varepsilon=8$;
    \item \texttt{private\_eps\_5}, a private batched selection method with privacy/noise parameter $\varepsilon=5$. 
\end{itemize}

For the private methods, we use the same batched update structure as in the proof-of-concept implementation and replace exact maximization with exponential-mechanism-style sampling over the finite hypothesis class. The batch size is selected according to the theory.

\textbf{Reporting metric: plateau episode.}
We summarize each regret curve by its \textbf{plateau episode}. We define the plateau episode as the first episode at which \textbf{95\% of the final cumulative regret} has already been incurred. This gives a robust scalar summary of how quickly the regret curve enters its flat regime.

\textbf{Plateau results.}
The resulting plateau episodes are reported in Table~\ref{tab:poc-plateau}, and Figure~\ref{fig:poc-regret} plots the cumulative regret curves for the easy and hard proof-of-concept environments. The curves show the same qualitative pattern as the plateau summary: non-private learning reaches the stable regime fastest, batching introduces a moderate delay, and stronger privacy noise leads to slower stabilization.

\begin{table}[!ht]
    \centering
    \begin{tabular}{ccccc}
        \toprule
        Environment & non-private no-batch & non-private batched & private $\varepsilon=8$ & private $\varepsilon=5$ \\
        \midrule
        Easy & 13 & 24 & 37 & 127 \\
        Hard & 31 & 34 & 87 & 218 \\
        \bottomrule
    \end{tabular}
    \vspace{1em}
    \caption{Plateau episode for each method and environment.}
    \label{tab:poc-plateau}
\end{table}

\begin{figure}[!ht]
    \centering
    \begin{minipage}{0.47\textwidth}
        \centering
        \includegraphics[width=\linewidth]{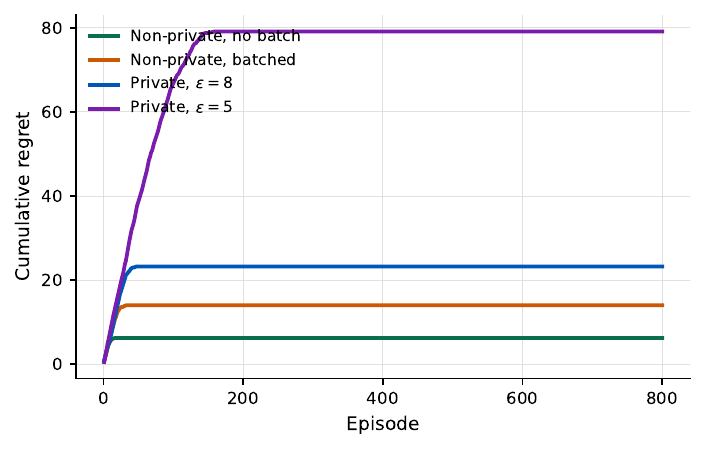}\\[-0.5em]
        {\small (a) Easy environment}
    \end{minipage}
    \hfill
    \begin{minipage}{0.47\textwidth}
        \centering
        \includegraphics[width=\linewidth]{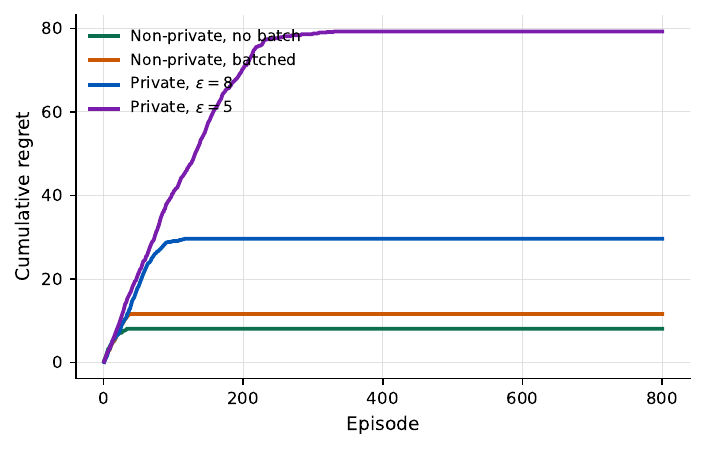}\\[-0.5em]
        {\small (b) Hard environment}
    \end{minipage}

    \caption{Cumulative regret curves in the proof-of-concept deterministic outcome-reward environments.}
    \label{fig:poc-regret}
\end{figure}

\textbf{Interpretation.} These results show the expected trend clearly. In both environments, the non-private methods plateau fastest. Among the private methods, the larger privacy parameter, $\varepsilon=8$, reaches the plateau regime substantially earlier than $\varepsilon=5$. The hard environment requires more episodes than the easy environment across all methods. In particular, in the easy environment, the plateau time increases from 37 episodes at $\varepsilon=8$ to 127 episodes at $\varepsilon=5$. In the hard environment, it increases from 87 episodes at $\varepsilon=8$ to 218 episodes at $\varepsilon=5$.

Thus, even in this minimal proof-of-concept simulation, weaker privacy constraints lead to faster convergence, while batching introduces a visible but smaller additional slowdown compared with the non-batched non-private baseline.

\textbf{Summary.}
This experiment is intentionally small-scale, but it directly validates the qualitative privacy-utility tradeoff emphasized in the paper: as privacy becomes stronger, more episodes are required to reach the stable learning regime. It also shows that the deterministic outcome-reward setting with general nonlinear finite function approximation already exhibits the expected separation among non-private, batched non-private, and private learning.

%% file: sections/app-proof.tex
\section{Proof of Section~\ref{sec:alg}}

\subsection{Proof of Theorem~\ref{thm:general}}
\label{app:general}
\input{sections/app-proof-general}

\subsection{Proof of Theorem~\ref{thm:deter}}
\label{app:deter}
\input{sections/app-proof-det}

%% file: sections/app-proof-general.tex
\begin{proof}

Denote $V^{\star}(s_1) = Q_1^\star(s_1) := \max_{a \in \cA} Q^\star(s_1, a)$, by the Exponential Mechanism, with probability at least $1-\alpha$, we have $S(f^{(k)}) \ge \max_{f \in \cF} S(f) - \frac{2\Delta_S}{\varepsilon_0} \log(|\cF| K /\alpha)$, which implies
\begin{align*}
    \eta f_1^{(k)}(s_1) - L_{\cD_{[k]}}^{\mathsf{BE}}(f^{(k)}) 
    \ge
    \eta Q_1^{\star}(s_1) - L_{\cD_{[k]}}^{\mathsf{BE}}(Q^{\star}) - \frac{2\Delta_S}{\varepsilon_0} \log(|\cF| K /\alpha).
\end{align*}

Dividing both sides by $\eta$ and rearranging terms yields
\begin{align*}
    V^{\star}(s_1) \le f_1^{(k)}(s_1) - \frac{1}{\eta}\left( L_{\cD_{[k]}}^{\mathsf{BE}}(f^{(k)}) - L_{\cD_{[k]}}^{\mathsf{BE}}(Q^{\star}) \right) + \frac{2\Delta_S}{\eta \varepsilon_0} \log(|\cF| K /\alpha).
\end{align*}

By subtracting $V^{\pi_k}(s_1)$ from both sides, we obtain
\begin{align*}
    V^{\star}(s_1) - V^{\pi_k}(s_1)
    \le 
    f_1^{(k)}(s_1) - V^{\pi_k}(s_1) - \frac{1}{\eta}\left( L_{\cD_{[k]}}^{\mathsf{BE}}(f^{(k)}) - L_{\cD_{[k]}}^{\mathsf{BE}}(Q^{\star}) \right) + \frac{2\Delta_S}{\eta \varepsilon_0} \log(|\cF| K /\alpha).
\end{align*}

For the term $f_1^{(k)}(s_1) - V^{\pi_k}(s_1)$, given by the standard performance difference lemma~\cite{Kakade2002ApproximatelyOA}, we have
\begin{align*}
    f_1^{(k)}(s_1) - V^{\pi_k}(s_1) 
    = \sum_{h=1}^{H} \mathbb{E}_{\pi_k} \left[f_h^{(k)}(s_h, a_h) - [\cT_h f_{h+1}^{(k)}](s_h, a_h) \right].
\end{align*}

Thus, denote $e_h^{(k)}(s,a) = f_h^{(k)}(s,a) - [\cT_h f_{h+1}^{(k)}](s, a)$, it holds that for $k \in [K]$,
\begin{align*}
    V^\star(s_1) - V^{\pi_{k}}(s_1) 
    &\le 
    \sum_{h=1}^H \mathbb{E}_{\pi_{k}} [e_h^{(k)}(s_h, a_h)] - \frac{L_{\cD_{[k]}}^{\mathsf{BE}}(f^{(k)}) - L_{\cD_{[k]}}^{\mathsf{BE}}(Q^\star)}{\eta} + \frac{2\Delta_S}{\eta \varepsilon_0} \log(|\cF| K /\alpha).
\end{align*}

Then, under the success event of Proposition 11 in~\citet{chen2025outcomebasedonlinereinforcementlearning}, we have
\begin{align*}
    V^\star(s_1) - V^{\pi_{k}}(s_1) - \frac{2H\kappa}{\eta}
    \le 
    \sum_{h=1}^H \left( \mathbb{E}_{\pi_{k}} [e_h^{(k)}(s_h, a_h)] - \frac{1}{2\eta} \sum_{t < [k]} \mathbb{E}_{\pi_{t}} e_h^{(k)}(s_h, a_h)^2 \right) + \frac{2\Delta_S}{\eta \varepsilon_0} \log(|\cF| K /\alpha),
\end{align*}
where $\kappa = \log \frac{\abs{\cF}}{\alpha} + \log \frac{H}{\alpha}$.

\begin{lemma} \label{lem:general_bound_original}

Suppose that $p_1,\ldots,p_K$ is a sequence of distributions over $\cX$ and there exists $\mu \in \Delta(\cX)$ such that $\frac{p_k(x)}{\mu(x)} \le C_{\mathsf{cov}}$ for all  $x \in \cX$ and $k \in [K]$. Then, for any sequence $f^{(1)},\ldots,f^{(K)}$ of functions $\cX\to[0,1]$ and any constant $\lambda \ge 1$, it holds that
\begin{align*}
    \sum_{k=1}^K \mathbb{E}_{x\sim p_k}\, f^{(k)}(x)
    \le
    C_{\mathsf{cov}}B \log_2 \left( 1 + \frac{C_{\mathsf{cov}}K}{\lambda} \right) + \sqrt{ 4C_{\mathsf{cov}}\log \left(1+\frac{C_{\mathsf{cov}}K}{\lambda}\right) \left[ 2K\lambda +\sum_{k=1}^K \sum_{t<[k]} \mathbb{E}_{x\sim p_t}\, f^{(k)}(x)^2 \right]}.
\end{align*}
\end{lemma}

Consider $\cX = \cS \times \cA$, and we define $p_{k} := \mathbb{P}^{\pi_{k}} \left( (s,a) = \cdot \right) \in \left(\cS \times \cA \rightarrow [0, 1] \right)$, assume that $\frac{p_{k}(s,a)}{\mu(s,a)} \le C_{\mathsf{cov}}$, then simply apply the Lemma~\ref{lem:general_bound_original}, we derive the following bound
\begin{align*}
        \sum_{k=1}^{K} \mathbb{E}_{\pi_{k}}  \abs{e_h^{(k)}(s_h, a_h)}
        \le
        C_{\mathsf{cov}}B\log_2\left( 1 + \frac{C_{\mathsf{cov}}K}{\lambda} \right) + \sqrt{4 C_{\mathsf{cov}} \log \left(1 + \frac{C_{\mathsf{cov}}K}{\lambda}\right)\Big[ 2K\lambda + \sum_{k=1}^{K}\sum_{t < [k]} \mathbb{E}_{\pi_{t}} e_h^{(k)}(s_h, a_h)^2\Big] }.
\end{align*}

Using Cauchy-Schwarz in the form $xy \le \frac{\eta}{2}x^2 + \frac{1}{2\eta} y^2$, we have
\begin{align*}
    \sum_{k=1}^K \mathbb{E}_{\pi_{k}} \abs{e_h^{(k)}(s_h, a_h)} 
    &\le C_{\mathsf{cov}} B \log_2 \left( 1 + \frac{C_{\mathsf{cov}} K}{\lambda} \right) \\ 
    &\quad + 2\eta C_{\mathsf{cov}} \log \left( 1 + \frac{C_{\mathsf{cov}} K}{\lambda} \right) + \frac{1}{2\eta} \left[ 2K \lambda + \sum_{k=1}^K \sum_{t<[k]} \mathbb{E}_{\pi_{t}} e_h^{(k)}(s_h, a_h)^2 \right].
\end{align*}

Therefore, combining the inequalities above, set $ \kappa = \lambda$ and take summation over $k \in [K]$ gives
\begin{align*}
    \sum_{k=1}^K \left( V^\star(s_1) - V^{\pi_{k}}(s_1) \right) 
    &\le \frac{3KH\kappa}{\eta} + \frac{2K\Delta_S}{\eta \varepsilon_0} \log(|\cF| K /\alpha) \\
    &\quad + 2\eta H C_{\mathsf{cov}} \log \left( 1 + \frac{C_{\mathsf{cov}}K}{\kappa} \right) + H C_{\mathsf{cov}} B \log_2 \left( 1 + \frac{C_{\mathsf{cov}} K}{\kappa} \right) .
\end{align*}

For the non-private case, we set
\begin{align*}
    \frac{3KH\kappa}{\eta} = C \eta \implies \eta = \sqrt{\frac{3KH\kappa}{C}}
\end{align*}
where $C = H C_{\mathsf{cov}} \log \left( 1 + \frac{C_{\mathsf{cov}}K}{\kappa} \right)$.
Substituting back, the non-private bound is $\frac{3KH\kappa}{\eta} + C \eta = 2\sqrt{3KH\kappa C}$.

For the private case, we write
\begin{align*}
    \frac{2K\Delta_S \log(|\cF| K /\alpha)}{\eta \varepsilon_0} = C \cdot \eta.
\end{align*}

\begin{lemma}[Advanced Sequential Composition~\cite{dwork2010boosting}] \label{lem:dp-advance}
Suppose $D = \cup_{i=1}^k D_i$ and each $\cM_i$ is $\varepsilon$-DP, then $M$ is $(\varepsilon', \delta')$-DP for any $\delta' \geq 0$ with
\begin{align*}
    \varepsilon' = \varepsilon \sqrt{2k \ln(1/\delta')} + k\varepsilon(e^\varepsilon - 1)
\end{align*}
\end{lemma}

For the privacy parameter $\varepsilon_0$, utilizing the \emph{Advanced Sequential Composition} (Lemma~\ref{lem:dp-advance}), we have:
\begin{align*}
    \varepsilon = \sqrt{2\frac{K}{B} \log (1/\delta)} \cdot \varepsilon_0 \implies \varepsilon_0 = \frac{\varepsilon}{\sqrt{2\frac{K}{B} \log (1/\delta)}},
\end{align*}
note that our Exponential Mechanism always provides pure-DP ($\varepsilon_0, 0$)-DP, the $\delta$ here is given by Advanced Sequential Composition.

Substituting $\varepsilon_0$ back into the expression for $\eta$, we obtain
\begin{align*}
    \eta
    = \sqrt{\frac{2K \Delta_S \log(|\cF| K /\alpha)}{C \varepsilon_0}}
    = \sqrt{\frac{2K \Delta_S \log(|\cF| K /\alpha)}{C \cdot \frac{\varepsilon \sqrt{B}}{\sqrt{2K \log(1/\delta)}}}}
    = O\left(\frac{K^{3/4}\Delta_S^{1/2}(\log(|\cF| K /\alpha))^{1/2} (\log (1/\delta))^{1/4}}{B^{1/4} \varepsilon^{1/2} C^{1/2}}\right).
\end{align*}

Finally, substituting back to the regret bound, we balance these two terms:
\begin{align*}
    O\left(\frac{K^{3/4}C^{1/2}\Delta_S^{1/2}(\log(|\cF| K /\alpha))^{1/2} (\log (1/\delta))^{1/4}}{B^{1/4} \varepsilon^{1/2}}\right) 
    = 
    O(B \cdot C),
\end{align*}

Thus we solve for $B = \frac{K^{3/5}\Delta_S^{2/5}(\log(|\cF| K /\alpha))^{2/5} (\log (1/\delta))^{1/5}}{\varepsilon^{2/5} C^{2/5}}$.

We now turn to bound the sensitivity $\Delta_S$, note that we have $ S(f) := \eta f_1(s_1) - L_{\cD}^{\mathsf{BE}}(f)$. Since the first term $\eta f_1(s_1)$ does not depend on $\cD$, so the sensitivity is entirely from $L_{\cD}^{\mathsf{BE}}$.

Recall the definition of $L_{\cD}^{\mathsf{BE}}$:
\begin{align*}
    L_{\cD}^{\mathsf{BE}}(f):= \sum_{h =1}^H \cE_{\cD,h}(f_h, f_{h+1}) - \inf_{f' \in \cF} \sum_{h=1}^H \cE_{\cD,h}(f'_h, f_{h+1}),
\end{align*}
where Bellman error $\cE_{\cD,h}(f_h, f_{h+1})$ at each $h$ is given by
\begin{align*}
    \sum_{\tau \in \cD} \left(f_h(s_h, a_h) - r_h - \max_{a'} f_{h+1}(s_{h+1},a')\right)^2.
\end{align*}

Since the rewards are normalized ($\sum_{h=1}^H r_h  \in [0,1]$), then
\begin{align*}
    \left(f_h(s_h, a_h) - r_h - \max_{a'} f_{h+1}(s_{h+1},a')\right) \in [-2,1]
\end{align*}

thus changing one trajectory changes $\sum_{h =1}^H \cE_{\cD,h}(f_h, f_{h+1})$ by at most $4H$. For the second term, since for a fixed $f'$, the sensitivity has the same bound of $4H$, and the $\inf_{f'}$ cannot increase the sensitivity. Therefore, $\Delta_S \le 8H$.

We now introduce the so-called \emph{billboard lemma}.
\begin{lemma}[Billboard Lemma~\cite{hsu2016}] \label{lem:billboard}
    Suppose $\cM: \cD \rightarrow \cR$ is $(\varepsilon, \delta)$-differentially private. Consider any set of functions $f_i : \cD_i \times \cR \rightarrow \cR'$, where $\cD_i$ is the portion of the database containing $i$'s data. The composition \{$f_i(\Pi_i \cD, \cM(\cD))$ \} is $\varepsilon, \delta$-joint differentially private, where $\Pi_i : \cD \rightarrow \cD_i$ is the projection to i's data.
\end{lemma}

In each episode $k$, the agent releases a function $f^{(k)}$ selected via the Exponential Mechanism with scoring function $S(f) = \eta f_1(s_1) - L_{\cD}^{\mathsf{BE}}(f)$. Given the sensitivity $\Delta_S \le 8H$ and the application of advanced sequential composition, the sequence of functions $\{f^{(k)}\}_{k=1}^K$ is $(\varepsilon, \delta)$-DP. 
The action $a_h^k$ recommended to the $k$-th user is a function of the DP-protected $f^{(k)}$ and the user's private state $s_h^k$. Since the actions sent to all other users $j \neq k$ depend on $\cD_k$ only through the $(\varepsilon, \delta)$-DP billboard $f^{(k)}$, the mechanism is $(\varepsilon, \delta)$-joint differentially private.

Setting
\begin{align*}
    \eta &= \max \Big\{ \sqrt{\frac{3KH\kappa}{C}},\frac{K^{3/4}\Delta_S^{1/2}(\log(|\cF| K /\alpha))^{1/2} (\log (1/\delta))^{1/4}}{B^{1/4} \varepsilon^{1/2} C^{1/2}}\Big\} \\
    B &= \frac{K^{3/5}\Delta_S^{2/5}(\log(|\cF| K /\alpha))^{2/5} (\log (1/\delta))^{1/5}}{\varepsilon^{2/5} C^{2/5}},
\end{align*}
we have the final bound
\begin{align*}
    \sum_{k=1}^{K} \left( V^\star(s_1) - V^{\pi_{k}}(s_1) \right)
    \le
    2\sqrt{3KH\kappa C} + \frac{K^{3/5} C^{3/5} \Delta_S^{2/5}(\log(|\cF| K /\alpha))^{2/5} (\log (1/\delta))^{1/5}}{\varepsilon^{2/5} },
\end{align*}
where $C = H C_{\mathsf{cov}} \log \left( 1 + \frac{C_{\mathsf{cov}}K}{\kappa} \right)$, $\kappa = \log \frac{\abs{\cF}}{\alpha} + \log \frac{H}{\alpha} $, $\Delta_S \le 8H$.

\end{proof}

\begin{proof}[Proof of Lemma~\ref{lem:general_bound_original}]
    Fix any $x \in \cX$, we divide the $M$ batches into two disjoint sets of ``good'' and ``bad'' batches, i.e., $[M] = \cG_x \cup \cB_x$, and $ \cG_x \cap \cB_x = \emptyset$, where both sets depend on the current $x$. Let $t_m:=(m-1)B+1$ denote the first episode of batch $m$, with the convention $t_{M+1}:=K+1$. To define the ``bad'' batch, we let $P_m(x) := \lambda \mu(x) + \sum_{t \le mB} p_t(x)$ with $P_0(x) := \lambda \mu(x)$. Then we define
    \begin{align*}
        \cB_x = \{m \in [M] : P_m(x) \ge 2P_{m-1} (x) \} 
    \end{align*}

    We now write the summation $\sum_{k=1}^K \mathbb{E}_{x\sim p_k}\, f^{(k)}(x)$ as follows.
    \begin{align} \label{eq:define_bad_batch}
        \sum_{k=1}^K \mathbb{E}_{x \sim p_k} f^{(k)}(x)
        &= \sum_{x \in \cX} \sum_{k=1}^{K} p_k(x)f^{(k)}(x)  \nonumber \\
        &= \underbrace{\sum_{x \in \cX} \sum_{m \in \cB_x} \sum_{k=t_{m}}^{t_{m+1}-1}p_k(x)f^{(k)}(x)}_{\cT_{\mathsf{bad}}} + \underbrace{\sum_{x \in \cX} \sum_{m \in \cG_x} \sum_{k=t_{m}}^{t_{m+1}-1}p_k(x)f^{(k)}(x)}_{\cT_{\mathsf{good}}}.
    \end{align}

    For $\cT_{\mathsf{bad}}$, we first bound the size of $\cB_x$ for any $x$ as follows.
    \begin{equation} \label{eq:b_x}
        \abs{\cB_x} \log 2
        \le \sum_{m \in [M]} \log \left(\frac{P_m(x)}{P_{m-1}(x)}\right) 
        = \log \left( \frac{\lambda \mu(x) + \sum_{t=1}^K p_t(x)}{\lambda\mu(x)} \right)
        \le \log \left( 1+ \frac{C_{\mathsf{cov}}K}{\lambda} \right).
    \end{equation}

    Hence, we have for any $x \in \cX$,
    \begin{align*}
        \abs{\cB_x} \le \log_2 (1 + \frac{C_{\mathsf{cov}}K}{\lambda}).
    \end{align*}

    With this result, we can turn to bound $\cT_{\mathsf{bad}}$ below
    \begin{align*}
        \cT_{\mathsf{bad}} 
        \stackrel{(a)}{\le} \sum_{x \in \cX} \sum_{m \in \cB_x} \sum_{k=t_m}^{t_{m+1}-1} p_k(x) \stackrel{(b)}{\le} \sum_{x \in \cX} \sum_{m \in \cB_x} \sum_{k=t_m}^{t_{m+1}-1} C_{\mathsf{cov}}\mu(x) 
        \stackrel{(c)}{\le} C_{\mathsf{cov}} \cdot B \cdot \log_2 \left(1+ \frac{C_{\mathsf{cov}}K}{\lambda} \right),
    \end{align*}

    where $(a)$ holds by the boundedness of $f^{(k)}$; $(b)$ follows from the definition of $C_{\mathsf{cov}}$; $(c)$ is true by~\eqref{eq:b_x}.

    Now, it remains to bound $\cT_{\mathsf{good}}$. To this end, we first define $\cG_x^k$ to be all the time slots that belong to a good batch. Moreover, for any $x \in \cX$ and $k \in [K]$, define
    \begin{align*}
        \tilde{p}_k(x) &:= \lambda \mu(x) + \sum_{t < [k]} p_t(x) \\
        \hat{p}_k (x) &:= \lambda \mu(x) + \sum_{t \le k} p_t(x)
    \end{align*}

    Thus, we have
    
    \begin{align*}
        \cT_{\mathsf{good}} &= \sum_{x \in \cX} \sum_{k \in \cG_x^k} p_k(x) f^{(k)}(x) \\
        &= \sum_{x \in \cX} \sum_{k \in \cG_x^k} \frac{p_k(x)}{\sqrt{\tilde{p}_k(x)}} \cdot f^{(k)}(x) \sqrt{\tilde{p}_k(x)} \\
        &\stackrel{(a)}{\le} \sum_{x \in \cX} \sum_{k \in \cG_x^k} \frac{\sqrt{2p_k(x)}}{\sqrt{\hat{p}_k(x)}} \cdot f^{(k)}(x) \sqrt{\tilde{p}_k(x)} \\
        &\stackrel{(b)}{\le} \sum_{x \in \cX} \sum_{k \in [K]} \frac{\sqrt{2p_k(x)}}{\sqrt{\hat{p}_k(x)}} \cdot f^{(k)}(x) \sqrt{\tilde{p}_k(x)} \\
        &\stackrel{(c)}{\le} \sqrt{2} \sqrt{\underbrace{\sum_{k \in [K]} \sum_{x \in \cX} \frac{p_k(x)^2}{\hat{p}_k(x)}}_{\cT_1}} \cdot \sqrt{\underbrace{\sum_{k \in [K]} \sum_{x \in \cX} \tilde{p}_k(x) f^{(k)}(x)^2}_{\cT_2}},
    \end{align*}
    where $(a)$ holds by the fact that for any $k$ in a good batch under $x$, $\hat{p}_k(x) \le 2\tilde{p}_k(x)$ as a result of the definition of the good batch~\eqref{eq:define_bad_batch}; $(b)$ holds by adding more positive terms; $(c)$ holds by Cauchy--Schwarz inequality.
    
    For $\cT_2$, by the definition of $\tilde{p}_k(x)$ and boundedness of $f^{(k)}$, we have
    \begin{align*}
    \sum_{k=1}^K \sum_{x \in \cX} \tilde{p}_k(x) f^{(k)}(x)^2 
    &\le \sum_{k=1}^K (\lambda + 1) + \sum_{k=1}^K \sum_{t < [k]} \mathbb{E}_{x \sim p_t} f^{(k)}(x)^2 \\
    &\le 2K\lambda + \sum_{k=1}^K \sum_{t < [k]} \mathbb{E}_{x \sim p_t} f^{(k)}(x)^2.
    \end{align*}

    For $\cT_1$, we note that
    \begin{align*}
        \cT_1 
        = \sum_{x \in \cX} \sum_{k=1}^K \frac{p_k(x)^2}{\hat{p}_k(x)} 
        &\le \sum_{x \in \cX} C_{\mathsf{cov}}\mu(x) \sum_{k=1}^K \frac{p_k(x)}{\hat{p}_k(x)} \\
        &\stackrel{(a)}{\le} 2 \sum_{x \in \cX} C_{\mathsf{cov}}\mu(x) \sum_{k=1}^K \log \left( 1 + \frac{p_k(x)}{\hat{p}_k(x)} \right) \\
        &\le 2 \sum_{x \in \cX} C_{\mathsf{cov}}\mu(x) \sum_{k=1}^K \log \left( 1 + \frac{p_k(x)}{\lambda\mu(x) + \sum_{t < k} p_t(x)} \right) \\
        &= 2 \sum_{x \in \cX} C_{\mathsf{cov}}\mu(x) \log \left( \frac{\lambda\mu(x) + \sum_{t \le K} p_t(x)}{\lambda\mu(x)} \right) \\
        &\le 2 C_{\mathsf{cov}} \log \left( 1 + \frac{C_{\mathsf{cov}}K}{\lambda} \right),
    \end{align*}
    where $(a)$ holds by the fact that $u \le 2\log(1+u)$ for any $u \in [0, 1]$. Thus, with results for $\cT_1$ and $\cT_2$, we have the following bound for $\cT_{\mathsf{good}}$
    \begin{align*}
        \cT_{\mathsf{good}} 
        \le \sqrt{4C_{\mathsf{cov}} \log \left( 1 + \frac{C_{\mathsf{cov}}K}{\lambda} \right) \cdot \left[ 2K\lambda + \sum_{k=1}^K \sum_{t < {[k]}} \mathbb{E}_{x \sim p_t} f^{(k)}(x)^2 \right]}.
    \end{align*}

    Finally, combining the bounds for $\cT_{\mathsf{bad}}$ and $\cT_{\mathsf{good}}$, we have the final bound
    \begin{align*}
        \sum_{k=1}^K \mathbb{E}_{x \sim p_k} f^{(k)}(x) 
        \le C_{\mathsf{cov}} B \log_2 \left( 1 + \frac{C_{\mathsf{cov}}K}{\lambda} \right) + \sqrt{4C_{\mathsf{cov}} \log \left( 1 + \frac{C_{\mathsf{cov}}K}{\lambda} \right) \cdot \left[ 2K\lambda + \sum_{k=1}^K \sum_{t < {[k]}} \mathbb{E}_{x \sim p_t} f^{(k)}(x)^2 \right]}
    \end{align*}
    
\end{proof}

%% file: sections/app-proof-det.tex
\begin{proof}
    For any $f \in \cF$, define
    \begin{align*}
        \bar{f}_1 := \mathbb{E}_{s \sim \rho} \max_a f_1(s,a).
    \end{align*}

    In Algorithm~\ref{alg:deterministic}, the score function is 
    \begin{align*}
        S(f) := \eta \bar{f}_1 - L_{\cD}^{\mathsf{BR}}(f).
    \end{align*}

    Let $[k]$ denote the beginning of the batch containing episode $k$, and let $\cD_{[k]}$ be the dataset available at that time. By the exponential mechanism and a union bound over all batch updates, with probability at least $1-\alpha$, for every episode $k$,
    \begin{align*}
        \eta \bar{f}_1^{(k)} - L_{\cD_{[k]}}^{\mathsf{BR}}(f^{(k)}) \ge \eta \bar{Q}^{\star}_1 - L_{\cD_{[k]}}^{\mathsf{BR}}(Q^\star) - \frac{2\Delta_S}{\varepsilon_0} \log(|\cF|K/\alpha),
    \end{align*}
    where $\bar{Q}^{\star}_1 := \mathbb{E}_{s_1\sim\rho}\max_{a\in\cA}Q^\star_1(s_1,a) = J^\star$.

    Rearranging, we have
    \begin{align*}
        J^{\star} \le \bar{f}^{(k)}_1 - \frac{1}{\eta} \left( L_{\cD_{[k]}}^{\mathsf{BR}}(f^{(k)}) - L_{\cD_{[k]}}^{\mathsf{BR}}(Q^{\star}) \right) + \frac{2 \Delta_S}{\eta \varepsilon_0} \log(|\cF|K/\alpha).
    \end{align*}

    Thus we subtracting $J(\pi_k) = \mathbb{E}_{s_1 \sim \rho} V^{\pi_k} (s_1)$ and obtain
    \begin{align*}
        J^{\star} - J(\pi_k) \le \bar{f}_1^{(k)} - J(\pi_k) - \frac{1}{\eta} \left( L_{\cD_{[k]}}^{\mathsf{BR}}(f^{(k)}) - L_{\cD_{[k]}}^{\mathsf{BR}}(Q^{\star}) \right) + \frac{2 \Delta_S}{\eta \varepsilon_0} \log(|\cF|K/\alpha).
    \end{align*}

    Since $\pi_k$ is greedy with respect to $f^{(k)}$, we have
    \begin{align*}
        \bar{f}_1^{(k)} - J(\pi_k) &= \mathbb{E}_{s_1 \sim \rho, \pi_k} \left[ \sum_{h=1}^H \left( f_h^{(k)}(s_h, a_h) - f_{h+1}^{(k)}(s_{h+1}) - R_h(s_h, a_h) \right) \right] \\
        &= \mathbb{E}_{s_1 \sim \rho, \pi_k} \left[ R^{(f^{(k)})}(\tau) - r(\tau) \right].
    \end{align*}

    Denote $e^{(k)}(\tau) := R^{(f^{(k)})}(\tau) - r(\tau)$, we have
    \begin{align*}
        \bar{f}_1^{(k)} - J(\pi_k) = \mathbb{E}_{s_1 \sim \rho, \pi_k} \left[ e^{(k)}(\tau) \right].
    \end{align*}

    Recall that 
    \begin{align*}
        L_{\cD_{[k]}}^{\mathsf{BR}}(f) := \sum_{t < [k]} \left( R^{(f)}(\tau^{(t)}) - r^{(t)} \right)^2,
    \end{align*}
    where $(\tau^{(t)}, r^{(t)})$ denotes the data collected in episode $t$.

    By the uniform convergence of the square loss(Proposition 19 in~\cite{chen2025outcomebasedonlinereinforcementlearning} )there exists an event of probability at least $1-\alpha$ such that, for all $k$,
    \begin{align} \label{eq:det_uc}
        \frac{1}{2} \sum_{t < [k]} \mathbb{E}_{s_1 \sim \rho, \pi_t} \left[ e^{(k)}(\tau)^2 \right] \le L_{\cD_{[k]}}^{\mathsf{BR}}(f^{(k)}) - L_{\cD_{[k]}}^{\mathsf{BR}}(Q^\star) + \kappa,
    \end{align}
    where $\kappa = O \left( H^3 \log \frac{|\cF| K}{\alpha} \right)$.

    Rearranging terms and applying~\eqref{eq:det_uc}, we obtain
    \begin{align*}
        J^{\star} - J(\pi_k) \le \frac{\kappa}{\eta} + \frac{2\Delta_S}{\varepsilon_0 \eta} \log (|\cF|K / \alpha) + \mathbb{E}_{s_1\sim\rho,\pi_k} \left[ e^{(k)}(\tau) \right] - \frac{1}{2\eta} \sum_{t < [k]} \mathbb{E}_{s_1 \sim \rho, \pi_t} \left[ e^{(k)}(\tau)^2 \right].
    \end{align*}

    Summing over $k=1, \ldots, K$,
    \begin{align*}
        \sum_{k=1}^K \left( J^{\star} - J(\pi_k) \right) &\le \frac{K \kappa}{\eta} + \frac{2K\Delta_S}{\varepsilon_0 \eta} \log (|\cF|K / \alpha) + \sum_{k=1}^K \mathbb{E}_{s_1 \sim \rho, \pi_k} \left[ e^{(k)}(\tau) \right] - \frac{1}{2 \eta} \sum_{k=1}^K \sum_{t < [k]} \mathbb{E}_{s_1 \sim \rho, \pi_t} \left[ e^{(k)}(\tau)^2 \right] \\
        &= \frac{K \kappa}{\eta} + \frac{2K\Delta_S}{\varepsilon_0 \eta} \log (|\cF|K / \alpha) + \mathbb{E}_{s_1 \sim \rho} \left[ \sum_{k=1}^K \mathbb{E}_{\pi_k} \left[ e^{(k)}(\tau) \middle| s_1 \right] - \frac{1}{2 \eta} \sum_{k=1}^K \sum_{t < [k]} \mathbb{E}_{\pi_t} \left[ e^{(k)}(\tau)^2 \middle| s_1 \right] \right].
    \end{align*}

    \begin{lemma} \label{lem:deterministic_bound}
    Let $\lambda \ge 1$. For any initial state $s_1 \in \cS$, consider the sequence $\{f^{(k)},\pi_k\}_{k=1}^K$ generated by Algorithm~\ref{alg:deterministic}, where $\pi_k=\pi_{f^{(k)}}$ is the greedy policy with fixed deterministic tie-breaking. Then
        \begin{align*}
            \sum_{k=1}^{K} \abs{\mathbb{E}_{\pi_k} \left[e^{(k)}(\tau) \middle| s_1 \right]} &\le \frac{BN(s_1)}{2 \log 2} \log \left(1+\frac{KH}{\lambda} \right)  \\
            &\quad + 2\sqrt{2N(s_1)\log \left(1+\frac{KH}{\lambda} \right)} \cdot \sqrt{ 4K\lambda + \sum_{k=1}^{K} \sum_{t<[k]} \left(\mathbb{E}_{\pi_{t}} \left[e^{(k)}(\tau) \middle| s_1 \right] \right)^2},
        \end{align*}
    where $N(s_1)$ is the total number of reachable state-action pairs from $s_1$ under policy class $\Pi_{\cF}$.
    \end{lemma}

    For a fixed $s_1$, applying Lemma~\ref{lem:deterministic_bound}, Jensen's inequality $\left( \mathbb{E}_{\pi_t} \left[e^{(k)}(\tau) \middle| s_1 \right] \right)^2 \le \mathbb{E}_{\pi_t} \left[ e^{(k)}(\tau)^2 \middle| s_1 \right]$ and Cauchy-Schwarz in the form $xy \le \eta x^2 + \frac{1}{4\eta}y^2$, we have
    \begin{align*}
        \sum_{k=1}^K \mathbb{E}_{\pi_k} \left[ e^{(k)}(\tau) \middle| s_1 \right] \le \frac{BN(s_1)}{2\log 2} \log \left( 1 + \frac{KH}{\lambda} \right) + \eta N(s_1) \log \left( 1 + \frac{KH}{\lambda}  \right) + \frac{K \lambda}{\eta} + \frac{1}{2\eta} \sum_{k=1}^K \sum_{t < [k]} \mathbb{E}_{\pi_t} \left[ e^{(k)}(\tau)^2 \middle| s_1 \right].
    \end{align*}

    Combining the inequalities above and setting $\lambda=\kappa$, we obtain
    \begin{align*}
        \sum_{k=1}^K \left( J^{\star} - J(\pi_k) \right) \le \frac{2K \kappa}{\eta} + \frac{2K\Delta_S}{\varepsilon_0 \eta} \log (|\cF|K / \alpha) + \mathbb{E}_{s_1 \sim \rho} \left[ \frac{BN(s_1)}{2 \log 2} \log \left( 1 + \frac{KH}{\kappa} \right) + \eta N(s_1) \log \left( 1 + \frac{KH}{\kappa} \right) \right].
    \end{align*}

    For each $s \in \cS$ and $h \in [H]$, define
    \begin{align*}
        S_h(s;\Pi_{\cF}) := \left\{ (s',a) : \exists \pi  \in \Pi_{\cF}, \text{ under } \pi \text{ and } s_1 = s, \text{ it holds that } s_h = s', a_h=a \right\}.
    \end{align*}
    and let $N_h(s; \Pi_{\cF}) := \abs{S_h(s; \Pi_{\cF})}$ and $N(s) := \sum_{h=1}^H N_h(s;\Pi_{\cF})$. 

    Under deterministic MDPs, given $s_1$, every greedy policy in $\Pi_\cF$ with fixed deterministic tie-breaking reaches a single element in $S_h(s_1;\Pi_\cF)$ at each layer $h$. Hence the layer-wise coverability of $M_{s_1}^\star$ is characterized by $N_h(s_1;\Pi_\cF)$, and therefore
    \begin{align*}
        N(s_1)
        \le
        H \cdot C_{\mathsf{cov}}(\Pi_\cF;M_{s_1}^{\star}).
    \end{align*}
    
    Taking expectation over $s_1 \sim \rho$, we derive
    \begin{align*}
        \mathbb{E}_{s_1 \sim \rho}[N(s_1)]
        \le
        H \cdot C'_{\mathsf{cov}}.
    \end{align*}
    
    Substituting back, with $C := H C'_{\mathsf{cov}} \log \left( 1 + \frac{KH}{\kappa} \right)$, we have 
    \begin{align} \label{eq:deterministic_regret}
        \sum_{k=1}^K \left( J^{\star} - J(\pi_k) \right) \le \frac{2K \kappa}{\eta} + \frac{2K\Delta_S}{\varepsilon_0 \eta} \log (|\cF|K / \alpha) + \eta \cdot C + B \cdot C.
    \end{align}

    For the non-private term, we balance the standard regret term and the regularization term:
    \begin{align*}
        \frac{2K\kappa}{\eta} = C \eta \quad \rightarrow \quad \eta = \sqrt{\frac{2K \kappa}{C}}.
    \end{align*}

    Substituting this back, the non-private part is 
    \begin{align*}
        \frac{2K \kappa}{\eta} + \eta C = O \left( \sqrt{K \kappa C} \right).
    \end{align*}

    For the private case, we balance
    \begin{align*}
        \frac{K \Delta_S}{\eta \varepsilon_0} \log (|\cF|K / \alpha) = C \eta.
    \end{align*}

    Given by the \emph{Advanced Sequential Composition} (Lemma~\ref{lem:dp-advance}), we consider the worst-case privacy loss: 
    \begin{align*}
        \varepsilon = O \left( \sqrt{\frac{K}{B} \log (1/\delta)}  \cdot \varepsilon_0 \right).
    \end{align*}

    Equivalently,
    \begin{align*}
        \varepsilon_0 = \Omega \left( \frac{\varepsilon \sqrt{B}}{\sqrt{K \log (1/\delta)}} \right).
    \end{align*}

    Substituting back, we obtain
    \begin{align*}
        \eta = O \left( \frac{K^{3/4} \Delta_S^{1/2} \left( \log (|\cF|K/\alpha) \right)^{1/2} \left( \log (1/\delta) \right)^{1/4}}{B^{1/4} \varepsilon^{1/2}C^{1/2}} \right).
    \end{align*}

    Finally, substituting this choice of $\eta$ back to~\eqref{eq:deterministic_regret}, the privacy contribution is
    \begin{align*}
        O \left( \frac{K^{3/4} C^{1/2} \Delta_S^{1/2} \left( \log (|\cF|K/\alpha) \right)^{1/2} \left( \log (1/\delta) \right)^{1/4}}{B^{1/4} \varepsilon^{1/2}} \right).
    \end{align*}
    
    We balance this term with the batching term $B \cdot C$:
    \begin{align*}
        O \left( \frac{ K^{3/4} C^{1/2} \Delta_S^{1/2} \left( \log (|\cF|K/\alpha) \right)^{1/2} \left( \log (1/\delta) \right)^{1/4}}{B^{1/4}\varepsilon^{1/2} } \right) = O(B \cdot C),
    \end{align*}
    which gives 
    \begin{align*}
        B = O \left( \frac{K^{3/5} \Delta_S^{2/5} \left( \log(|\cF|K/\alpha) \right)^{2/5} \left( \log(1/\delta) \right)^{1/5} }{\varepsilon^{2/5} C^{2/5}} \right).
    \end{align*}

    We now turn to bound the sensitivity $\Delta_S$. Note that
    \begin{align*}
        S_{\cD}(f) = \eta \bar{f}_1 - L_{\cD}^{\mathsf{BR}}(f).
    \end{align*}

    Since the first term $\eta \bar{f}_1$ does not depend on $\cD$, the sensitivity is entirely from $L_{\cD}^{\mathsf{BR}}(f)$. Recall
    \begin{align*}
        L_{\cD}^{\mathsf{BR}}(f) = \sum_{(\tau, r) \in \cD} \left( R^{(f)}(\tau) - r \right)^2.
    \end{align*}

    For each $h$, we have $f_h(s_h, a_h) \in [0,1]$ and $f_{h+1}(s_{h+1}) \in [0,1]$, so
    \begin{align*}
        R^{(f)}(\tau) = \sum_{h=1}^H \left[ f_h(s_h, a_h) - f_{h+1}(s_{h+1}) \right] \in [-H, H].
    \end{align*}

    Since $r \in [0,1]$, it holds that
    \begin{align*}
        \abs{R^{(f)}(\tau) - r} \le H+1.
    \end{align*}

    Consequently, changing one trajectory changes $L_{\cD}^{\mathsf{BR}}(f)$ by at most $(H+1)^2$, and we conclude that
    \begin{align*} 
        \Delta_S = O(H^2).
    \end{align*}

    Similarly, by applying the \emph{billboard lemma} (Lemma~\ref{lem:billboard}), Algorithm~\ref{alg:deterministic} satisfies $(\varepsilon, \delta)$-JDP.

    Setting
    \begin{align*}
        \eta = \max \left\{ \sqrt{\frac{2K \kappa}{C}}, \frac{K^{3/4}\Delta_S^{1/2}\left( \log (|\cF|K/\alpha) \right)^{1/2} \left( \log(1/\delta) \right)^{1/4}}{B^{1/4}\varepsilon^{1/2}C^{1/2}} \right\},
    \end{align*}
    and using the above choice of $B$, we have the final bound
    \begin{align*}
        \sum_{k=1}^K \left( J^{\star} - J(\pi_k) \right) \le O \left( \sqrt{K\kappa C} \right) + O \left( \frac{K^{3/5} C^{3/5} \Delta_S^{2/5} \left( \log (|\cF|K/\alpha) \right)^{2/5} \left( \log (1/\delta) \right)^{1/5} }{\varepsilon^{2/5}} \right).
    \end{align*}

    Plugging in
    \begin{align*}
        C = H \cdot C'_{\mathsf{cov}} \log \left( 1 + \frac{KH}{\kappa} \right), \quad \kappa = O\left( H^3 \log(|\cF|K/\alpha) \right), \quad \Delta_S = O(H^2),
    \end{align*}
    we obtain
    \begin{align*}
        \mathsf{Reg} \le 2H^2 \sqrt{2 C'_{\mathsf{cov}} K  \log(|\cF|K/\alpha) \log K} + H^{7/5}  \frac{ (C'_{\mathsf{cov}} K \log K)^{3/5} \left(\log (|\cF|K/\alpha )\right)^{2/5} (\log (1/\delta))^{1/5} }{\varepsilon^{2/5}}
    \end{align*}

\end{proof}

\begin{proof}[Proof of Lemma~\ref{lem:deterministic_bound}]

    Fix an initial state $s_1 \in \cS$. Define the index set of reachable state-action pairs
    \begin{align*}
        \cI:=\{(h,s,a): h\in[H], (s,a)\in S_h(s_1;\Pi_\cF) \} \subseteq[H] \times \cS \times \cA,
        \quad |\cI|=N(s_1).
    \end{align*}
    
    For each episode $k\in[K]$, define $\psi^{(k)}\in \mathbb{R}^{|\cI|}$ by
    \begin{align*}
         \psi^{(k)}_{(h,s,a)} := f_h^{(k)}(s,a) - f_{h+1}^{(k)}(s^+_h(s,a)) - R_h(s,a),
    \end{align*}
    where $s^+_h(s,a)$ is the unique next state reached from $(s,a)$ at step $h$ under the deterministic transition, and $f_{h+1}^{(k)}(s):=\max_{a'\in\cA}f_{h+1}^{(k)}(s,a')$ with the convention $f_{H+1}^{(k)}\equiv 0$.
    
    And define the occupancy vector $\phi^{(k)}\in \mathbb{R}^{|\cI|}$ by 
    \begin{align*}
        \phi^{(k)}_{(h,s,a)}:=\mathbb{P}_{\pi_{k}}(s_h=s, a_h=a | s_1).
    \end{align*}

    Then, by the definition of  $e^{(k)}(\tau)$ and the deterministic transition, the on-policy error in episode $k$ can be written as
    \begin{align*}
        \mathbb{E}_{\pi_{k}} \left[ e^{(k)}(\tau) \middle| s_1 \right]
        =\langle \psi^{(k)}, \phi^{(k)} \rangle.
    \end{align*}
    Thus, our target is to bound $\sum_{k=1}^{K} \left|\langle \psi^{(k)}, \phi^{(k)}\rangle\right|$.
    
    Define the covariance matrices based on historical visitations
    \begin{align*}
        V_{k}:=\lambda I + \sum_{j=1}^{k} \phi^{(j)} \phi^{(j)\top}.
    \end{align*}

    Since the policy is updated only at batch boundaries, the historical covariance matrix available at the beginning of the batch containing episode $k$ is
    \begin{align*}
        V_{[k]-1} = \lambda I + \sum_{t < [k]} \phi^{(t)}\phi^{(t)\top},
    \end{align*}
    where $[k]$ denotes the first episode of the batch containing episode $k$.
    
    Define the bad-round set $\Psi$ where the covariance matrix at the beginning of the batch significantly underestimates the real-time matrix
    \begin{align*}
        \Psi := \left\{ k\in[K]: \frac{ \| \phi^{(k)} \|_{V_{[k]-1}^{-1}}}{\| \phi^{(k)} \|_{V_{k-1}^{-1}}} >2 \right\}.
    \end{align*}
    We decompose the total sum into bad and good rounds
    \begin{align*}
        \sum_{k=1}^{K} \abs{\langle \psi^{(k)},\phi^{(k)}\rangle}
        =\sum_{k\in\Psi} \abs{\langle \psi^{(k)},\phi^{(k)}\rangle}
        +\sum_{k\notin\Psi} \abs{\langle \psi^{(k)},\phi^{(k)}\rangle}.
    \end{align*}
    
    For good rounds ($k \notin \Psi$), by definition $\|\phi^{(k)}\|_{V_{[k]-1}^{-1}}\le 2\|\phi^{(k)}\|_{V_{k-1}^{-1}}$. Hence, by the Cauchy--Schwarz inequality,
    \begin{align*}
        \| \langle \psi^{(k)},\phi^{(k)}\rangle \|  \le \| \psi^{(k)}\|_{V_{[k]-1}}\cdot \|\phi^{(k)}\|_{V_{[k]-1}^{-1}} 
        \le 2 \| \psi^{(k)} \|_{V_{[k]-1}} \cdot \|\phi^{(k)}\|_{V_{k-1}^{-1}}.
    \end{align*}
    
    Summing over good rounds and applying Cauchy--Schwarz again to the sum
    \begin{align*}
        \sum_{k\notin\Psi} \abs{\langle \psi^{(k)},\phi^{(k)}\rangle}
        &\le
        2\sqrt{\sum_{k=1}^{K}\min \left\{1, \|\phi^{(k)}\|_{V_{k-1}^{-1}}^2\right\}}
        \cdot
        \sqrt{\sum_{k=1}^{K} \max \left\{1, \|\psi^{(k)}\|_{V_{[k]-1}}^2\right\}}.
    \end{align*}
    
    By the elliptical potential lemma~\cite{lattimore2020bandit}, and using $\|\phi^{(k)}\|_2^2\le H$, the first term is bounded by
    \begin{align*}
        \sum_{k=1}^{K} \min \left\{1, \|\phi^{(k)}\|_{V_{k-1}^{-1}}^2\right\}
        \le 
        2N(s_1)\log \left(1+\frac{KH}{N(s_1)\lambda}\right).
    \end{align*}
    
    For the second term involving $\psi^{(k)}$, by the definition of $V_{[k]-1}$,
    \begin{align*}
        \|\psi^{(k)}\|_{V_{[k]-1}}^2 = \lambda\|\psi^{(k)}\|_2^2+\sum_{t < [k]}\langle \phi^{(t)},\psi^{(k)}\rangle^2 .
    \end{align*}

    Moreover, for every $t < [k]$,
    \begin{align*}
        \langle \phi^{(t)}, \psi^{(k)}\rangle = \mathbb{E}_{\pi_t}\left[e^{(k)}(\tau) \middle| s_1\right].
    \end{align*}

    By the definition of $\psi^{(k)}$, we have $\| \psi^{(k)} \|_{\infty} \le 2$ and obtain
    \begin{align*}
        \sum_{k=1}^{K}\max \left\{1,\|\psi^{(k)}\|_{V_{[k]-1}}^2\right\}
        \le 4K\lambda N(s_1) + \sum_{k=1}^{K}\sum_{t < [k]}\left(\mathbb{E}_{\pi_{t}} \left[e^{(k)}(\tau) \middle| s_1\right]\right)^2.
    \end{align*}

    Now we bound the bad rounds. For those bad rounds ($k\in\Psi$), since $\pi_k$ is greedy with respect to $f^{(k)}$, we have
    \begin{align*}
        \langle \psi^{(k)},\phi^{(k)}\rangle = \mathbb{E}_{\pi_k}[e^{(k)}(\tau) | s_1] = f_1^{(k)}(s_1)-V^{\pi_k}(s_1).
    \end{align*}

    By the normalization of rewards and the range of functions in $\cF$, both terms lie in $[0,1]$. Hence
    \begin{align*}
        \sum_{k\in\Psi} \abs{\langle \psi^{(k)},\phi^{(k)}\rangle}
        \le |\Psi|.
    \end{align*}

    We bound $|\Psi|$ using a determinant argument. Let $t_m := (m-1)B + 1$ denote the first episode of batch $m$, with the convention $t_{M+1} := K+1$. For any $k \in \Psi$, suppose $k$ belongs to batch $m$, so that $t_m = [k]$. By the definition of $\Psi$,
    \begin{align*}
        \frac{\|\phi^{(k)}\|_{V_{t_m - 1}^{-1}}}{\|\phi^{(k)}\|_{V_{k-1}^{-1}}} > 2.
    \end{align*}
    Since $V_{t_m -1}\preceq V_{k-1}$, the standard determinant comparison inequality implies
    \begin{align*}
        \log\left( \frac{\det(V_{k-1})}{\det(V_{t_m -1})} \right) > 2\log2.
    \end{align*}
    Moreover, since $V_{k-1}\preceq V_{t_{m+1} - 1}$, we have
    \begin{align*}
        \log\left( \frac{\det(V_{t_{m+1}-1})}{\det(V_{t_m-1})} \right)
        > 2\log2.
    \end{align*}
    
    Let $\cM_\Psi$ be the set of batch-start indices $m$ whose batch contains at least one bad round. Summing the above inequality over these bad batches and bounding it by the total log-determinant growth yields
    \begin{align*}
        |\cM_\Psi| \cdot (2\log 2)
        \le
        \sum_{m \in \cM_\Psi}
        \log\frac{\det(V_{t_{m+1}-1})}{\det(V_{t_m-1})}
        \le
        \log\frac{\det(V_K)}{\det(V_0)}.
    \end{align*}
    
    Applying the standard determinant bound
    \begin{align*}
        \log \frac{\det V_K}{ \det V_0} \le N(s_1) \log \left( 1 + \frac{KH}{ N(s_1)\lambda} \right),
    \end{align*}
    we can solve for the number of bad batches
    \begin{align*}
        |\cM_\Psi| \le \frac{N(s_1)}{2\log 2}\log \left(1+\frac{KH}{N(s_1)\lambda} \right).
    \end{align*}
    
    Since each batch has size $B$, the total number of bad rounds $|\Psi|$ is at most $B \cdot |\cM_\Psi|$. Therefore
    \begin{align*}
        |\Psi| \le \frac{B N(s_1)}{2\log 2}\log \left(1+\frac{KH}{N(s_1)\lambda}\right).
    \end{align*}
    
    Combining the bounds for good and bad rounds and rescaling $\lambda \leftarrow \frac{\lambda}{N(s_1)}$, we obtain
    \begin{align*}
        \sum_{k=1}^{K} \left|\mathbb{E}_{\pi_{k}} \left[e^{(k)}(\tau) \middle| s_1 \right]\right|
        &\le \frac{BN(s_1)}{2\log 2} \log \left(1+\frac{KH}{\lambda} \right) \\
        &\quad + 2\sqrt{2N(s_1)\log \left(1+\frac{KH}{\lambda} \right)} \cdot \sqrt{4K\lambda + \sum_{k=1}^{K}\sum_{t < [k]}\left(\mathbb{E}_{\pi_{t}}[e^{(k)}(\tau) \middle| s_1]\right)^2 }.
    \end{align*}
    This completes the proof.
\end{proof}